\documentclass{article} 
\usepackage{iclr2026_delta,times}
\usepackage{amsmath,amssymb,mathtools}
\usepackage{algorithm,algpseudocode}
\usepackage[dvipsnames]{xcolor}
\usepackage{amsmath, amssymb}
\usepackage{algpseudocode}
\usepackage{algorithm}
\usepackage{amsthm}
\usepackage{wrapfig}
\usepackage[hidelinks]{hyperref}

\usepackage{subcaption}
\usepackage{tabularx}
\usepackage{comment}
\usepackage{booktabs,multirow}
\usepackage{makecell}

\PassOptionsToPackage{final}{graphicx}
\usepackage{graphicx}

\usepackage{amsmath,amsfonts,bm}

\def\eqref#1{equation~\ref{#1}}

\def\1{\bm{1}}

\DeclareMathAlphabet{\mathsfit}{\encodingdefault}{\sfdefault}{m}{sl}
\SetMathAlphabet{\mathsfit}{bold}{\encodingdefault}{\sfdefault}{bx}{n}

\newcommand{\Ls}{\mathcal{L}}
\newcommand{\R}{\mathbb{R}}

\newcommand{\sigmoid}{\sigma}

\DeclareMathOperator*{\argmin}{arg\,min}

\DeclareMathOperator{\Tr}{Tr}

\newtheorem{proposition}{Proposition}

\newcommand{\bt}{{\bm{\theta}}}
\newcommand{\gen}{ \bm{G}_{\bt} }
\newcommand{\bhat}[1]{\hat{\bm{#1}}}  
\newcommand{\bigO}{\mathcal{O}}

\usepackage{hyperref}
\usepackage{url}

\newcommand{\SQL}{\textsc{neuroSQL} }

\title{Generative Model via Quantile Assignment}

\author{Georgi Hrusanov$^{1,2}$\thanks{Corresponding author.},\hspace{1mm} 
Oliver Y. Ch\'en$^{1,2}$,\hspace{1mm} 
and Julien S. Bodelet$^{1,2,3}$%
\\[1mm]
$^{1}$University of Lausanne, Lausanne, Switzerland\\[0.5mm]
$^{2}$Lausanne University Hospital (CHUV), Lausanne, Switzerland\\[0.5mm]
$^{3}$Department of Statistics, Stanford University, Stanford, USA\\[1mm]
\texttt{georgi.hrusanov@chuv.ch}\\
\texttt{olivery.chen@chuv.ch}\\
\texttt{julien.bodelet@chuv.ch}
}

\iclrfinalcopy 

\begin{document}

\maketitle

\begin{abstract}
  Deep Generative models (DGMs) play two key roles in modern machine learning: (i) producing new information (e.g., image synthesis) and (ii) reducing dimensionality. 
  However, traditional architectures often rely on auxiliary networks such as encoders in Variational Autoencoders (VAEs) or discriminators in Generative Adversarial Networks (GANs), which introduce training instability, computational overhead, and risks like mode collapse.
In this work, we present NeuroSQL, a new generative paradigm that eliminates the need for auxiliary networks by learning low-dimensional latent representations implicitly.
NeuroSQL leverages an asymptotic approximation that expresses the latent variables as the solution to an optimal transportation problem.
Specifically, NeuroSQL learns the latent variables by solving a linear assignment problem and then passes the latent information to a standalone generator. 
  To demonstrate NeuroSQL's efficacy, we benchmark its performance against GANs, VAEs, and a budget-matched diffusion baseline on four independent datasets on handwritten digits (MNIST), faces from the CelebFaces Attributes Dataset (CelebA), animal faces from Animal Faces HQ (AFHQ), and brain images from the Open Access Series of Imaging Studies (OASIS). Compared to VAEs, GANs, and diffusion models: (1) in terms of image quality, NeuroSQL achieves overall lower mean pixel distance between synthetic and authentic images and stronger perceptual/structural fidelity, under the same computational setting; (2) computationally, NeuroSQL requires the least amount of training time; and (3) practically, NeuroSQL provides an effective solution for generating synthetic data when there are limited training data (e.g., data with a higher-dimensional feature space than the sample size). Taken together, by embracing quantile assignment rather than an encoder, NeuroSQL provides a fast, stable, and robust way to generate synthetic data with minimal information loss.

\end{abstract}

\section{Introduction}

Deep generative models (DGMs) have become a cornerstone of machine learning and have made meaningful contributions to image synthesis, data augmentation, and creative content generation. Over the past decade, they have also become well established in diverse scientific fields, such as genomics and neuroimaging, for complex data analysis tasks, including data interpretation, image/sequence decoding, and the generation of realistic datasets. A large share of these advances has been powered by variational autoencoders (VAEs; \citealp{Kingma2014}) and generative adversarial networks (GANs; \citealp{goodfellow2014generative}), which remain the two dominant paradigms for generative modeling from lower-dimensional latent spaces. Both frameworks adopt a common strategy: pairing a generator with a complementary deep neural network (DNN). In VAEs, an encoder maps observations to latent variables, whereas in GANs, a discriminator provides adversarial feedback to train the generator.

Despite their promise, training DGMs remains challenging, with practical and conceptual limitations. Practically, optimizing multiple networks simultaneously can be unstable and may lead to failures such as mode collapse. Training auxiliary networks that operate directly on high-dimensional observations is often more difficult and data-intensive than training the generator itself, which typically takes low-dimensional latent inputs. GANs, in particular, are prone to convergence failures such as mode collapse \citep{mescheder2018training}. VAEs often produce blurred reconstructions due to variational approximations and pixel-wise reconstruction losses such as mean squared error (MSE). More broadly, joint training of multiple deep networks increases sample complexity, computational cost, and instability, and these issues are exacerbated when data are high-dimensional relative to sample size.

Theoretically, there is no guarantee that the required encoder or discriminator, which works as a continuous mapping that a DNN can approximate, exists in the first place. While the generator depends on a latent variable that can have a much lower dimension than the observations, both the encoder and discriminator are functions of the data itself and may therefore encounter a curse of dimensionality \citep{stone1985additive}. Although DNNs are believed to achieve fast convergence rates \citep{schmidt2020nonparametric} or, under favorable conditions, mitigate the curse of dimensionality (see e.g. \citealp{suzuki2019adaptivity}), these properties rely on strict assumptions \citep{Lederer2025} that are not typically satisfied by DGMs. These limitations are reflected in theoretical work on VAEs and GANs, which often sidestep the resulting difficulties \citep{bodelet2025statistical,chae2023likelihood,biau2020some}.

More recently, Denoising Diffusion Probabilistic Models (DDPMs) have set new benchmarks for image quality through iterative sampling. Yet diffusion models remain computationally intensive and require significant training time. In resource-constrained or low-data regimes, they often suffer from high-variance score estimates and diluted gradients.

\begin{wrapfigure}{r}{0.52\textwidth}
  \vspace{-0.5\baselineskip} 
  \centering
  \includegraphics[width=\linewidth]{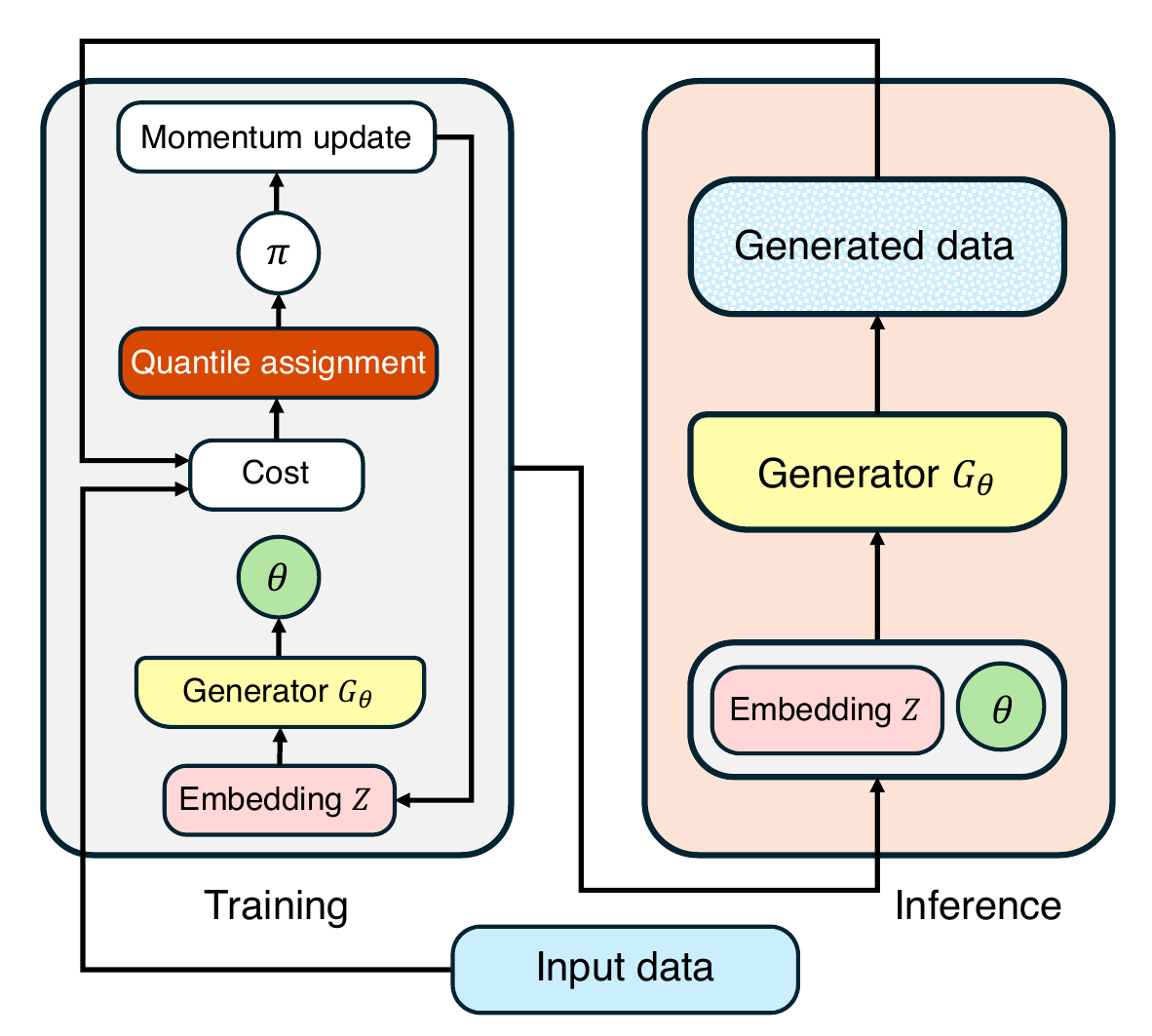}
  \caption{\textbf{A schematic representation of the \SQL architecture.}
  Left: Algorithm for optimizing latent embeddings and generator parameters $\theta$.
  Right: Data flow in \SQL for data synthesis.
  Here, ``cost'' refers to cost matrix entries:
  $C_{i,k} = \ell\big(X_i, G_\theta(Q_k)\big)$ solved by the linear assignment problem.
  ``Momentum update'' indicates:
  $\hat{Z}_{i}^{(t)} \leftarrow \rho\, Q_{\pi^{(t)}(i)} + (1-\rho)\, \hat{Z}_{i}^{(t-1)}$.}
  \label{fig:model-overview}
  \vspace{-1.5\baselineskip} 
\end{wrapfigure}

To address the limitations in existing methods, we propose NeuroSQL, a latent-variable DGM that does not rely on amortized inference (encoders), adversarial feedback (discriminators), or iterative sampling (diffusion). Instead, \SQL uses a discrete assignment mapping grounded in an asymptotic approximation: the unknown latent realizations can be represented, as $n\to\infty$, by a permutation of fixed (multi)variate quantiles of the assumed prior. This reduces latent estimation to identifying the permutation, which we show is identifiable and can be recovered by solving a linear assignment problem. In the multivariate setting, quantiles are constructed via an optimal-transport formulation, yielding a principled partition of the latent space into $n$ regions. \SQL generalizes Statistical Quantile Learning (SQL; \citealp{bodelet2025statistical}), which focused on additive models and univariate quantiles.

In its essence, \SQL trains a unique DNN for the generator, alongside an assignment procedure.
Compared to existing DGMs, \SQL offers several key benefits: 
It reduces the total parameter count, simplifies the training procedure to focus exclusively on the generator, avoids the curse of dimensionality typically encountered by auxiliary networks, and eliminates the instability challenges inherent in simultaneous multi-network optimization. 
Thus, \SQL is fast, stable, and resource-friendly, consistently delivering superior image quality under matched computational constraints, proving the conceptual benefits of this new paradigm.

\section{Methodology 
}
\label{sec:method}

\subsection{The model}

In what follows, we first state the model we aim to learn.
We consider a dataset $\mathcal{D}$ composed of $n$ independent samples of
$p$-dimensional random vectors, $\mathcal{D}=\{ \bm X_1, \dots, \bm X_n \}$.
We assume that the data are driven by unobserved continuous latent variables $\bm Z_i \in \mathcal{Z}\subseteq \R^d$. Specifically, we consider a probabilistic generative model:
\begin{equation}\label{eq. model}
\bm X_i = \bm G(\bm Z_i) + \bm \epsilon_i,
\end{equation}
where $\bm G:\mathcal{Z} \rightarrow \R^p$ is an unknown generator, and $\bm \epsilon_i$ are independent random errors.
We assume that the latent variables $\bm Z_i$ follow a known continuous distribution $P_Z$.
The latent dimension is typically (much) smaller than the ambient dimension ($d << p$) to enable dimensionality reduction.
This, therefore, contrasts with flow-based generative models, where the generator must be invertible.
As DGMs are not identifiable, one can select any distribution as long as it has a continuous cumulative distribution function.
Regarding the prior distribution, it is common to use either the standard normal distribution, $\mathcal{N}(\bm 0,\bm I)$, or a Uniform distribution.
We approximate the generator using deep neural networks (DNN),
that is $\bm G \approx \gen$, where:
$$
\gen \in \{ W_L \circ \sigmoid \circ W_{L-1} \circ \sigmoid \dots \circ \sigmoid  \circ W_{1} \},
$$
and $W_l$ denote affine transformations and $\sigmoid$ is an activation function.
The vector $\bt \in \Theta$ contains the DNN parameters.
We aim to learn both the generator parameter $\bt$ and the latent variables $\bm Z_i$.

\subsection{Latent space approximation and loss function}

For deep generative models, likelihood-based estimation is typically tractable only in simple settings (e.g., linear factor analysis). Jointly learning the generator and latent variables is, therefore, computationally demanding. In \SQL we adopt an approach inspired by the sieve method, which replaces an intractable optimization over a complex parameter space by tractable problems on a growing sequence of simpler subspaces that are dense in the original space (see, e.g., \cite{Chen2007}).

\SQL aims to approximate the latent space to obtain a tractable problem. It partitions the latent space into $n$ regions via quantiles $\bm Q^n_1,\dots,\bm Q^n_n\in\mathcal Z\subseteq\R^d$. Let $\bm Z:=(\bm Z_1,\dots,\bm Z_n)'$ and $\bm Q^n:=(\bm Q^n_1,\dots,\bm Q^n_n)'$ denote the $n\times d$ matrices of latent variables and quantiles, respectively. The key idea of \SQL is that there exists a permutation $\pi$ such that
\begin{equation}\label{eq. latent approximation}
\bm Z \approx \bm Q^n_{\pi},
\end{equation}
where, for a matrix (or column vector) $\bm A$ and a permutation $\pi$, $\bm A_{\pi}$ denotes the matrix obtained by permuting the rows of $\bm A$ according to $\pi$.

Since $\R^d$ has no canonical ordering, there is no universal definition of multivariate quantiles. For clarity, we formalize the approximation in \eqref{eq. latent approximation} for the univariate case ($d=1$); the $d$-dimensional construction is given in Section~\ref{sec. multivariate quantiles}.

For $d=1$, define the $n$ quantiles $\bm Q_i^n = F^{-1}\!\left(\frac{i}{n+1}\right)$, $i=1,\dots,n$, where $F$ is the CDF of $P_Z$. Using the delta method (see, e.g., \citealt{vandervaart2000asymptotic}), one has $\bigl|\bm Z_{(i)}-\bm Q_i^n\bigr|=\bigO_p\!\left(\frac{1}{\sqrt{n}}\right)$, where $\bm Z_{(i)}$ denotes the order statistics of the latent variables. Since the $\bm Z_{(i)}$'s are distinct almost surely, they are almost surely a permutation of the $\bm Z_i$'s, which implies the approximation error bound
\begin{equation}\label{eq. approximation error}
\min_{\pi \in S_n}\frac{1}{n}\sum_{i=1}^n \bigl|\bm Z_i-\bm Q^n_{\pi(i)}\bigr|^2
= \bigO_p\!\left(\frac{1}{n}\right),
\end{equation}
with $S_n$ the symmetric group of order $n$. Hence, the latent variables can be approximated by learning an appropriate permutation of the quantiles; the approximation error in \eqref{eq. approximation error} vanishes as $n\to\infty$.

Leveraging this approximation, we consider the criterion $\Ls(\bt,\pi)=\frac{1}{n}\sum_{i=1}^n \ell\!\left(\bm X_i,\gen\!\left(\bm Q^n_{\pi(i)}\right)\right)$ for a loss $\ell:\R^p\times\R^p\to\R_+$. We define the \SQL solutions by
\begin{equation}\label{eq:objective}
(\hat{\bt},\hat{\pi})
= \argmin_{\bt\in\Theta,\;\pi\in S_n}\; \Ls(\bt,\pi) \;+\; \lambda\,\mathcal R(\bt),
\end{equation}
where $\mathcal R(\bt)$ controls the complexity of $\gen$ and $\lambda>0$ is a tuning parameter. Given $\hat{\pi}$, the latent-variable estimator is $\bhat Z=\bm Q^n_{\hat{\pi}}$.

\subsection{\texorpdfstring{Multivariate quantiles ($d>1$)}{Multivariate quantiles (d>1)}}
\label{sec. multivariate quantiles}

Building on the univariate case, here we extend \SQL to $d>1$. We discuss how \SQL approximates the latent space in higher dimensions and show that the approximation error vanishes as $n \rightarrow \infty$. 
As there is no canonical ordering when $d>1$, 
several methods have been developed to construct ``multivariate quantiles''.
We concentrate on recent developments that leverage the optimal transport approach \citep{hallin2022measure, chernozhukov2017monge,hallin2021distribution,ghosal2022multivariate}, which offer a conceptually clean and practical way to define quantiles in multivariate dimensions. 

To build those multivariate quantiles, we use a regular grid, $\bm U_1, \bm U_2, \dots, \bm U_n \in \mathcal{U}_d$, where $\mathcal{U}_d:= \{\bm z \in \R^d: \Vert \bm z \Vert < 1\}$ is the unit ball.
This grid does not have to be perfectly regular, as it is in general not possible for $d\geq 3$.
We require only that the discrete distribution with probability $1/n$ at each grid point converges weakly to the uniform distribution over $\mathcal{U}_d$.
In practice, it is suitable to select a grid with a low discrepancy 
in order to obtain fast convergence rates.

In the case where the latent variables are uniformly distributed over $\mathcal{U}_d$, we define the multivariate quantiles as $\bm Q^n_i := \bm U_i$.
For more general distributions $F$, we define the multivariate quantiles as $\bm Q^n_i:= F_{\pm}^{-1}(U_i)$, where $F_{\pm}$ denotes the center-outward distribution function. Specifically, $F_\pm$ is the unique gradient of a convex function pushing $P_Z$ forward to the uniform distribution over the unit ball.
We refer to \cite{hallin2021distribution} and \cite{hallin2022measure} for a detailed explanation.
Without loss of generality, we will assume here that the distribution $P_Z$ is uniform over $\mathcal{U}_d$, yielding $\bm Q^n_i= \bm U_i$. 
This is reasonable because the latent distribution of (deep) generative models is not identifiable and should be selected (we refer to  \citealp{bodelet2025statistical} for a discussion).

The following proposition shows that the approximation error also vanishes asymptotically in the multivariate case.

\begin{proposition}\label{prop.}
Assume that the discrete distribution with probability $1/n$ at each grid point $\bm U_1, \bm U_2, \dots, \bm U_n \in \mathcal{U}_d$ converges weakly to the uniform distribution over $\mathcal{U}_d$. Then, as $n \rightarrow \infty$, the following holds:
$$ \min_{\pi \in S_n} \Vert \bm Z -\bm Q^n_\pi \Vert^2 \rightarrow 0, \text{ a.s. }$$
\end{proposition}

The proof of Proposition \ref{prop.} is in the Appendix \ref{app:proof}, following Theorem 2.4 
in \cite{hallin2021distribution}.



\subsection{Computational algorithm}
\label{subsec:algorithm}

We note that, for fixed $\bt$, the inner problem in Eq \ref{eq:objective} can be formulated as a \emph{linear assignment problem}:
\[
\min_{\pi\in S_n}\ \frac{1}{n}\sum_{i=1}^n C_{i,\pi(i)},\qquad C_{i,k}:=\ell\big(\bm X_i, \gen(\bm Q_k)\big),
\]
where $\bm C := (C_{i,k})_{1\leq i,k \leq n}$ is the cost matrix.
This is solvable \emph{exactly} by the Hungarian method in $O(n^3)$ time \citep{kuhn1955hungarian}, but see Sec \ref{sec:complexity} where one can reduce complexity to $O(n^2)$ and also do assignment in mini-batches.
Furthermore, for fixed $\pi$, Eq \ref{eq:objective} reduces to standard supervised regressions of $\bm X$ on assigned codes $\{z_{\pi(i)}\}$.
We therefore solve Eq \ref{eq:objective} as follows:
(i) Given a permutation $\pi$, minimize the loss function with respect to $\bt$ (Generator step);
(ii) Given $\bt$, we solve the linear assignment matching problem (via Hungarian $O(n^3)$ or Greedy method) which reduces the complexity to $O(n^2)$.
We iterate (i) and (ii) until convergence.
Furthermore, we introduce a momentum update after each assignment, that is 
$\widehat z^{(t)}\!\leftarrow\!\rho\,z_{\pi^{(t)}(i)}+(1-\rho)\,\widehat z^{(t-1)}$ for some $0\leq \rho < 1$, in order to stabilize training.
The exact steps are detailed in Algorithm \ref{alg:neurosql}.

\begin{algorithm}[t]
\caption{\SQL\ (full-batch assignment)}
\label{alg:neurosql}
\footnotesize
\begin{algorithmic}[1]
\State \textbf{Input:} data $\{\bm X_i\}_{i=1}^n$, prior $P_Z$, lattice $\bm Q^{n}$,
outer iters $T$, momentum $\rho\in[0,1)$, $\lambda>0$
\State Initialize $\pi^{(0)}$ (e.g., random or PCA-sorted);
\Statex \hspace{\algorithmicindent} set $\widehat{\bm Z}^{(0)} = (\bm Q^{n})_{\pi^{(0)}}$
\For{$t=1,\dots,T$}
  \State \textbf{Decoder step:}
  \Statex \hspace{\algorithmicindent}%
  $\bm\theta^{(t)} \in \arg\min_{\bm\theta}\Big[
    \begin{aligned}[t]
      &\frac{1}{n}\sum_{i=1}^n \ell\!\big(\bm X_i,\bm G_{\bm\theta}(\widehat{\bm Z}_i^{(t-1)})\big)
      + \lambda \mathcal R(\bm\theta)
    \end{aligned}
  \Big]$
  \State \textbf{Cost matrix:} $C_{i,k}^{(t)} \leftarrow \ell(\bm X_i,\bm G_{\bm\theta^{(t)}}(\bm Q_k))$
  \State \textbf{Assignment:} $\pi^{(t)} \leftarrow \arg\min_{\pi\in S_n}\Tr(\bm C_{\pi}^{(t)})$
  \State \textbf{Momentum:} $\widehat{\bm Z}_i^{(t)} \leftarrow
    \rho\,(\bm Q^{n})_{\pi^{(t)}(i)} + (1-\rho)\,\widehat{\bm Z}_i^{(t-1)}$
\EndFor
\State \textbf{Output:} $\widehat{\bm\theta}=\bm\theta^{(T)}$, $\widehat{\bm Z}=\widehat{\bm Z}^{(T)}$.
\end{algorithmic}
\end{algorithm}


\subsection{Complexity and scalability}
\label{sec:complexity}

Each outer iteration forms the cost matrix $C^{(t)}$ using $n$ decoder forward passes over ${z_k}$ and then solves a single assignment problem, for a total complexity of $O(n \cdot c_G + n^3)$, where $c_G$ is the decoder’s forward cost. Fundamentally, the assignment step is independent of the data dimension $p$, which is why \SQL scales favorably to very high-dimensional observations.

The $O(n^3)$ Hungarian algorithm can become a bottleneck with large $n$. To reduce this overhead (especially for $n \gtrsim 10^4$), we also propose a simple greedy assignment with $O(n^2)$ complexity (see Section~\ref{appendix:greedy} for details). In our benchmarks, the greedy strategy provides a $6.54\times$ speedup, reducing the mean time per call from $\approx 89$ms to $\approx 13$ms when loading the full batch, without a noticeable change in performance.

Across datasets, this substitution does not materially affect the qualitative quality of generated samples, making greedy assignment a practical choice in higher-throughput settings. For ultra-large datasets, one can further run \SQL in a mini-batch regime with batch size $m \ll n$, reducing assignment cost to $O(m^2)$ and making the computation of \SQL independent of the total dataset size.




\subsection{Interpretability of the Lattice Codes}

\begin{wrapfigure}{r}{0.46\textwidth}
  \centering
  \begin{subfigure}[b]{0.49\linewidth}
    \centering
    \includegraphics[scale=0.2, trim= 1.2cm 1.2cm 3cm 0.6cm, clip]{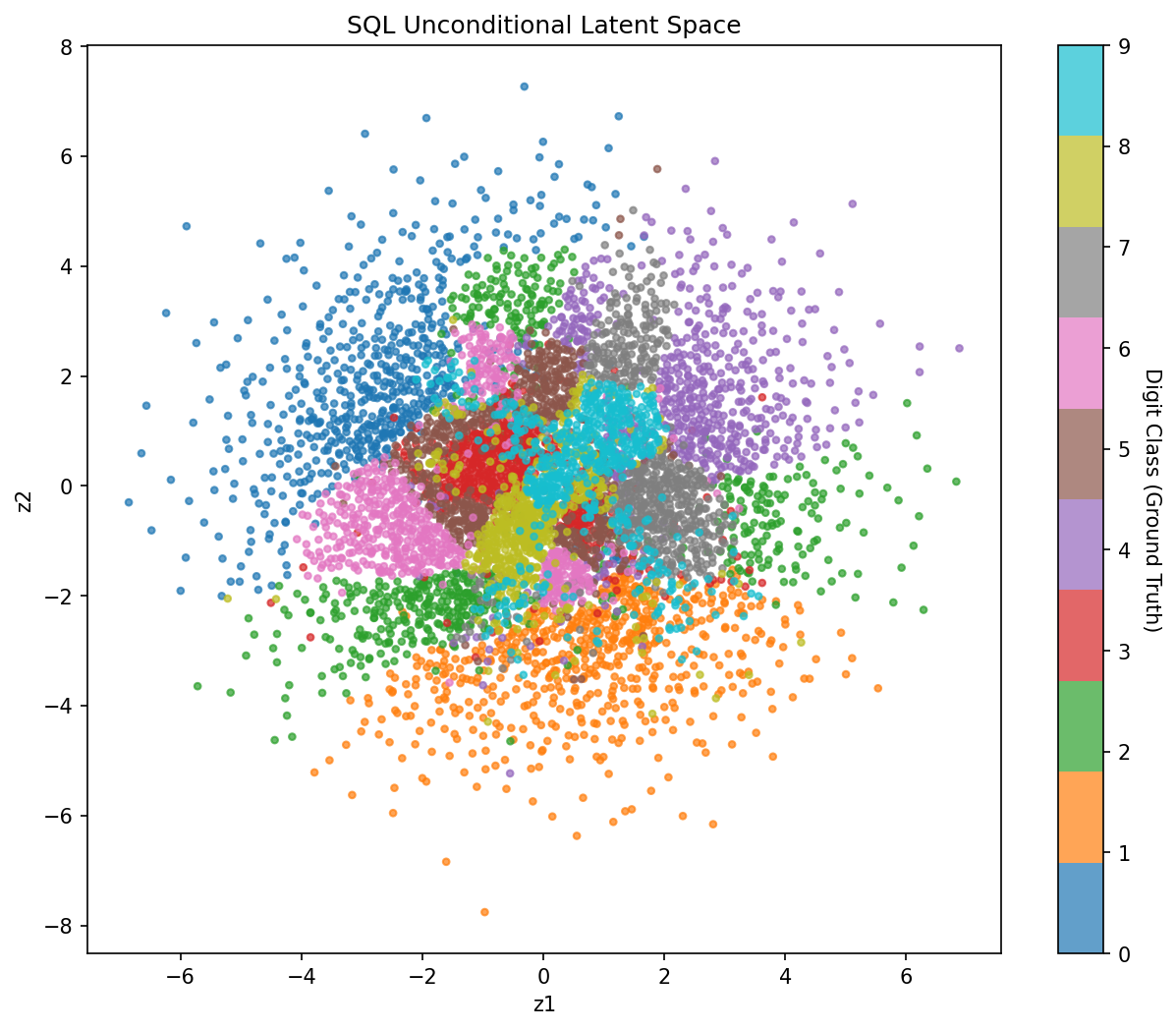}
    \caption{\SQL}
    \label{fig:mnist-neurosql-latent}
  \end{subfigure}\hfill
  \begin{subfigure}[b]{0.49\linewidth}
    \centering
    \includegraphics[scale=0.2, trim= 1.2cm 1.2cm 0.8cm 0.6cm, clip]{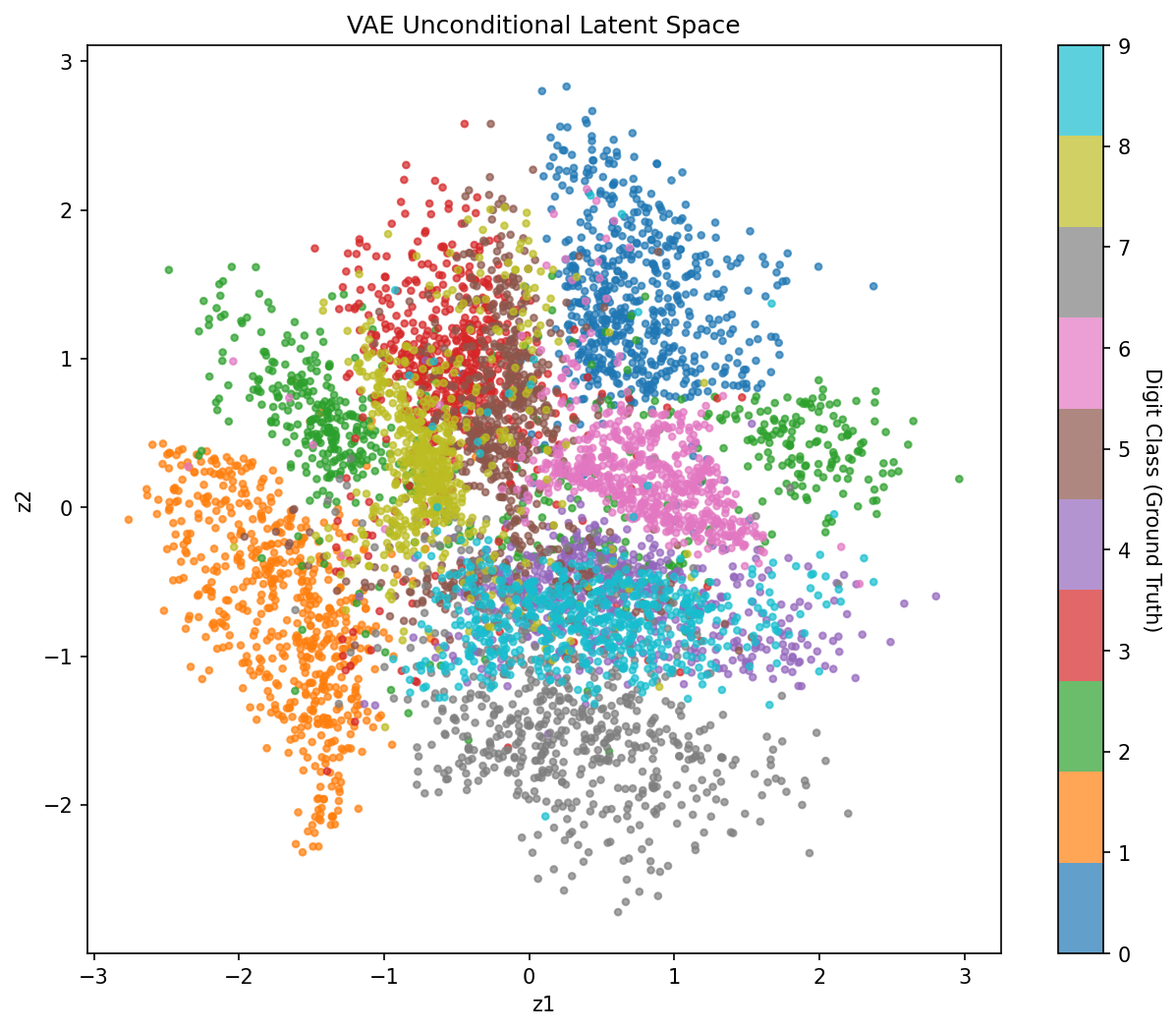}
    \caption{VAE}
    \label{fig:mnist-vae-latent}
  \end{subfigure}

  \caption{\textbf{A Comparison of Latent Space of \SQL and VAE.}
  Visualization of the two-dimensional latent space obtained from MNIST using \SQL (left) and VAE (right).
  NeuroSQL’s latent space forms more distinct, better-separated clusters than those of the VAE, whose clusters show more overlaps.}
  \label{fig:latent_interpretability}
    \vspace{-1.5\baselineskip}

\end{wrapfigure}

\SQL replaces the stochastic VAE encoder with a deterministic and explainable assignment procedure, hypothesizing that removing variational noise yields a more structured latent space. To probe this, we visualize 2D MNIST embeddings for \SQL and a VAE baseline with the same generator. As shown in Fig.~\ref{fig:latent_interpretability}b, the VAE latent space forms a characteristic “fuzzy cloud” under the Gaussian prior $\mathcal{N}(0,I)$: the KL regularization term crowds samples toward the origin, and class regions overlap substantially. This phenomenon is most prominent in the dense center, where digits such as 3, 5, and 8 become difficult to disentangle. Thus, a more precise latent analysis becomes harder.

In contrast, NeuroSQL (Fig.~\ref{fig:latent_interpretability}a) produces a more separated geometry. By approximating latents via a linear assignment to a fixed, high-variance lattice, it uses the space more evenly (roughly $[-8,8]$ versus $[-3,3]$ for the VAE) and organizes digit classes into compact, well-isolated “islands” (e.g., `1’ and `0’ form tight clusters with minimal spillover). This suggests that the Optimal-Transport-based assignment preserves semantic discreteness without supervision and reduces the overlap seen with the VAE, supporting a more transparent mapping between latent codes and generated samples.

\section{Experiments}
\label{sec:data-exp}
We evaluate \SQL under a \emph{sparse-resource} regime across four data domains of increasing structural complexity: digits (MNIST), human faces (CelebA), animal faces (AFHQ), and brain imaging (OASIS). More details on the data are in Section \ref{subsec:datasets}.

For comparison, we consider ConvNet, ResNet, and U\mbox{-}Net generators.
Our emphasis lies on \emph{paradigm evaluation}: models share (as closely as possible) the same generator backbone, data budgets, and optimization schedules, so that if there are any differences, they are likely to arise from the learning principle (quantile–assignment vs.\ probabilistic/adversarial/denoising) rather than model capacity or compute power.
As part of the evaluation pipeline, we train a downstream classifier using synthetic images.

\vspace{-2mm}
\subsection{Models and training}
\label{subsec:protocol-summary}



To ensure that performance differences are not attributable to architectural or training advantages,
we compare NeuroSQL, VAE, and GAN under the same three models: ConvNet(baseline), ResNet, and U-Net. Since Diffusion Probabilistic Models require a specific denoising architecture (usually MLP or a U-Net) and cannot use the controlled generator backbone(ConvNet/ResNet) shared across \SQL and the other two baselines, we matched the DDPM by scaling its width. This was achieved by implementing the DDPM with a custom denoising U-Net whose channel width was scaled to match the trainable parameter count of the competing generators within a $\pm10\%$ margin, ensuring comparable model capacity.

\begin{wrapfigure}{r}{0.52\textwidth}
  \centering
  \includegraphics[width=\linewidth]{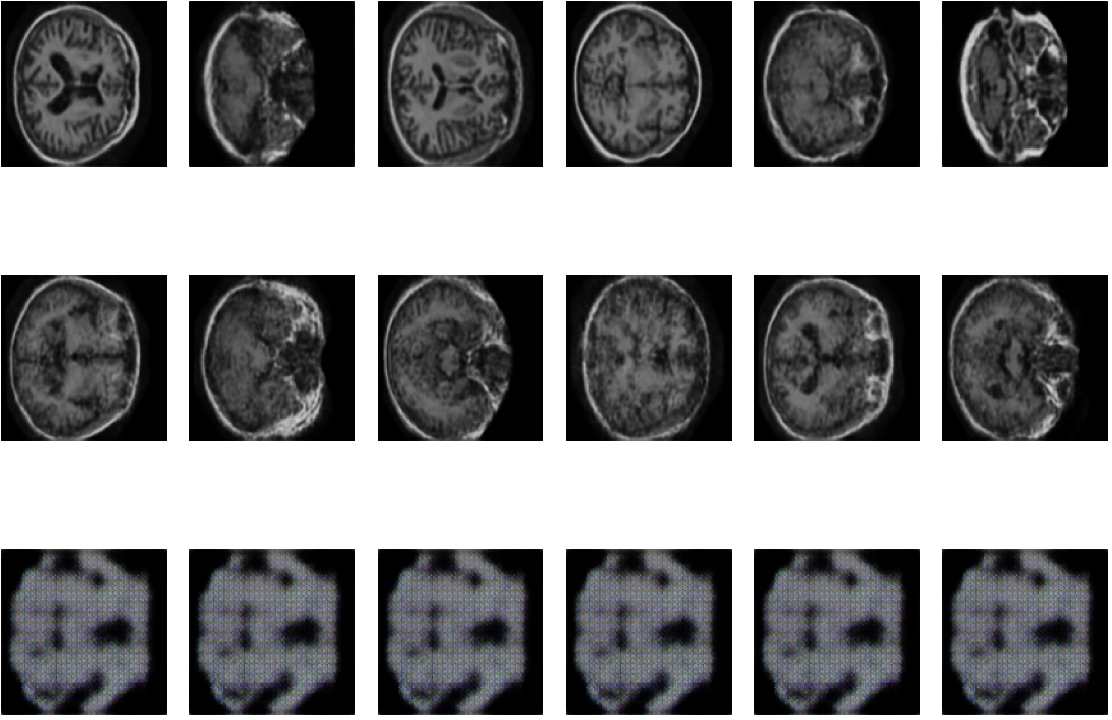}
  \caption{2D brain images generated from \SQL (top row), VAE (middle row) and GAN (bottom row) with U-Net generator.}
  \label{fig:oasis-generated}
      \vspace{-2.8\baselineskip} 

\end{wrapfigure}

For all methods, we matched computational budgets by fixing the total number of training iterations and using identical optimizer schedules (AdamW) and early stopping, thereby strictly controlling for both model size and compute time. Implementation details are in Appendix \ref{app:implementation details}.

\subsection{Evaluation}

To quantify the model performance, we report, for each scenario, a proxy of FID (Fréchet Inception Distance; lower is better), LPIPS (Learned Perceptual Image Patch Similarity; lower is better), and SSIM (Structural Similarity Index; higher is better). Discussions on the metrics and datasets are in Sections \ref{subsec:datasets} and \ref{subsec:metrics} of the Appendix. We present the results across a wide range of experimental specifications in Appendix~\ref{tab:results}.

\subsection{Results} 

For demonstration, we compare and discuss \textsc{NeuroSQL}'s performance against that of VAE and GAN on generating brain imaging data and MNIST digits. We include a comparison between \SQL and diffusion models as well as model performance of these models on human and animal faces in the Appendix.

For the brain imaging data (OASIS), we observe that, overall, \SQL outperforms both VAE and GAN across various combinations of latent dimension and generator. Particularly, \SQL outperforms VAE and GAN in terms of LPIPS (measuring perceptual similarity between images (i.e., how similar they look to humans) and SSIM (measuring pixel-level structural similarity) scores in all scenarios. All models with a ConvNet generator yield much better results compared to those with a ResNet. The choice of generator, however, has a marginal effect on FID scores. For FID, which measures distribution similarity between generated and real images in feature space and pixel distance, \SQL outperforms VAE and GAN in all cases with a ResNet generator. It outperforms VAEs and GANs when the latent dimension is moderate (between 16 and 64) with a ConvNet generator. When the latent dimension is very small or very large, VAE with ConvNet shows only modest improvement. With a U\mbox{-}Net generator, \SQL attains the best LPIPS and SSIM, while the FID (proxy) varies and is lowest for VAE on average.


Taken together, our results demonstrate the overall effectiveness of \SQL compared to existing methods in generating synthetic data across different domains that are unrelated to each other. In particular, \SQL with a ConvNet generator achieves superior performance across different scenarios, particularly in metrics evaluating perceptual similarity between images (how similar they look to humans) and pixel-level structural similarity, and its performance improves as the dimension of its latent space increases.

\begin{wraptable}{r}{0.58\textwidth}
  \centering
  \caption{Performance comparison across datasets (mean $\pm$ std). Lower is better for proxy FID; higher is better for SSIM.}
  \label{tab:main_results}
  \scriptsize
  \setlength{\tabcolsep}{3pt}
  \renewcommand{\arraystretch}{1.05}
  \begin{tabular}{llccc}
    \toprule
    \textbf{Dataset (Res.)} & \textbf{Model} &
    \textbf{pFID}$\downarrow$ & \textbf{SSIM}$\uparrow$ & \textbf{P} \\
    \midrule
    \textbf{MNIST (28$\times$28)} & \SQL      & \textbf{0.61$\pm$0.12} & \textbf{0.616$\pm$0.074} & \textbf{2.79} \\
                                 & VAE       & 1.26$\pm$0.26          & 0.179$\pm$0.030          & 4.17 \\
                                 & GAN       & 2.00$\pm$0.22          & 0.233$\pm$0.018          & 5.56 \\
                                 & Diffusion & 0.65$\pm$0.34          & 0.492$\pm$0.055          & 147.91 \\
    \addlinespace
    \textbf{CelebA (64$\times$64)} & \SQL      & \textbf{5.81$\pm$4.05} & \textbf{0.262$\pm$0.036} & \textbf{7.54} \\
                                  & VAE       & 10.75$\pm$7.92         & 0.196$\pm$0.047          & 14.03 \\
                                  & GAN       & 18.02$\pm$6.78         & 0.137$\pm$0.034          & 10.33 \\
                                  & Diffusion & 23.79$\pm$8.11         & 0.013$\pm$0.066          & 147.91 \\
    \addlinespace
    \textbf{AFHQ (128$\times$128)} & \SQL      & \textbf{19.03$\pm$11.31} & \textbf{0.290$\pm$0.056} & \textbf{119.58} \\
                                  & VAE       & 39.15$\pm$30.64          & 0.190$\pm$0.051          & 144.98 \\
                                  & GAN       & 46.00$\pm$14.83          & 0.082$\pm$0.029          & 122.37 \\
                                  & Diffusion & 22.99$\pm$14.83          & 0.0388$\pm$0.032         & 147.91 \\
    \addlinespace
    \textbf{OASIS (128$\times$128)} & \SQL      & \textbf{16.36$\pm$12.25} & \textbf{0.252$\pm$0.010} & \textbf{243.74} \\
                                   & VAE       & 24.24$\pm$24.46          & 0.196$\pm$0.058          & 269.09 \\
                                   & GAN       & 68.23$\pm$15.52          & 0.145$\pm$0.063          & 246.50 \\
                                   & Diffusion & 21.22$\pm$23.71          & 0.04759$\pm$0.058        & 252.66 \\
    \bottomrule
  \end{tabular}
\vspace{-2.0\baselineskip} 

\end{wraptable}

A more comprehensive quantitative comparison is detailed in Table \ref{tab:main_results}, which reports the mean performance metrics across all tested architectures and latent dimensions. These results confirm that \SQL delivers a robust balance of image quality and structural similarity across digits, faces, and brain imaging data with the smallest amount of trainable parameters.

\subsection{MNIST Downstream Classifier}

To further evaluate the quality of the generated samples from NeuroSQL, we train a downstream classifier using sampled images. We trained a standard CNN classifier solely on 10k synthetic images generated by \SQL and evaluated it on real MNIST data. The model achieved \textbf{87.71\% accuracy, 0.88 precision, and 0.88 recall}. Notably, our method significantly outperforms the VAE baseline, which achieved 67.49\% accuracy under the same conditions. Therefore, the generated samples likely effectively capture the class-conditional distributions of the underlying manifold.

\begin{table}[h]
\centering
\caption{Downstream classification performance on MNIST. A classifier was trained solely on synthetic data (10k samples) generated by \SQL and VAE and evaluated on real test data.}
\label{tab:downstream_mnist}
\small
\setlength{\tabcolsep}{4pt} 
\begin{tabular}{lccc}
\toprule
\textbf{Method} & \textbf{Prec.} & \textbf{Rec.} & \textbf{Acc.} \\
\midrule
\textsc{\textbf{NeuroSQL}} & \textbf{0.8788 }& \textbf{0.8763} & \textbf{0.8771} \\
VAE  & 0.7236 & 0.6705 & 0.6749 \\
\bottomrule
\end{tabular}
\end{table}

\vspace{-2mm}
\subsection{Compute budget, data scale, and evaluation scope} 
\label{subsec:budget}
\noindent

\begin{wrapfigure}{r}{0.62\textwidth}
  \centering

  \begin{subfigure}[b]{0.32\linewidth}
    \centering
    \includegraphics[width=\linewidth, trim=0 0 23cm 0cm, clip]{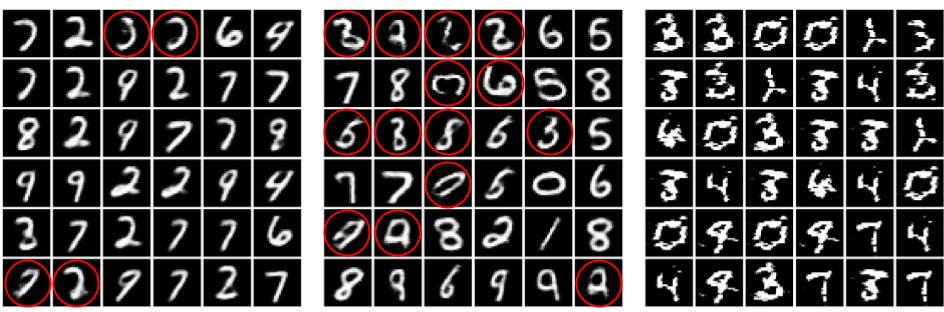}
    \caption{\SQL}
    \label{fig:mnist-neurosql}
  \end{subfigure}\hfill
  \begin{subfigure}[b]{0.32\linewidth}
    \centering
    \includegraphics[width=\linewidth, trim=11.5cm 0 11.5cm 0cm, clip]{figs/output-digits.png}
    \caption{VAE}
    \label{fig:mnist-neurosql-vae}
  \end{subfigure}\hfill
  \begin{subfigure}[b]{0.32\linewidth}
    \centering
    \includegraphics[width=\linewidth, trim=23cm 0 0 0cm, clip]{figs/output-digits.png}
    \caption{GAN}
    \label{fig:mnist-neurosql-gan}
  \end{subfigure}

  \caption{Qualitative comparison of 36 randomly generated images for models trained on MNIST.}
  \label{fig:mnist-generated}
    \vspace{-1.0\baselineskip}

\end{wrapfigure}

All experiments are designed to run end-to-end on a single Google Colab under a fixed allowance of \emph{200 compute units}. Within this budget, we cap the training set at 2000 images and focus on resolutions in the range $64{\times}64$–$128{\times}128$ (apart from MNIST), which in turn controls model capacity and the number of optimizer updates. Under these conditions, diffusion models underperform: with $T{=}1000$ diffusion steps and $N{\approx}1000$–$2000$ training images, supervision per noise level scales as $O(N/T)$, producing high-variance score estimates; moreover, training and sampling cost scale with $T$, which further reduces effective learning progress at a fixed wall-clock budget.

Despite the advantages of NeuroSQL, we highlight that our main goal with this work is a \textbf{paradigm-level proof of concept}. We aim not to compete with large-scale, high-resolution diffusion pipelines, but to demonstrate that \SQL provides a viable training principle for generative modeling in compute- and data-frugal regimes, such as in small research labs and non-profit organizations where pretraining, heavy augmentation, or distillation entail substantial challenges.

\section{Related Work}

\noindent\textbf{Discrete latent representations and vector quantization.} 
Discrete latent space models, including VQ-VAE~\citep{van2017neural}, VQ-VAE-2~\citep{razavi2019generating}, and VQ-GAN~\citep{esser2021taming}, utilize learned codebooks to map inputs to discrete vectors, a paradigm further advanced by masked~\citep{chang2022maskgit} and autoregressive~\citep{tian2024visual} modeling. While \textsc{NeuroSQL} also employs discrete representations, it fundamentally departs from this framework by replacing the learned encoder and its associated sample complexity with a fixed \textit{a priori} quantile lattice, determining embeddings via combinatorial assignment rather than gradient-based inference.

\textbf{Optimal Transport and Rectified Flows.}
A related and rapidly evolving line of work addresses this bottleneck by combining generative modeling with optimal transport (OT). 
Rectified Flows \citep{liu2022flow} formulate generative modeling as learning straight ODE trajectories between noise and data distributions, with subsequent work improving training efficiency \citep{lee2024improving}, achieving one-step generation \citep{geng2025meanflow}, and enabling forward-only regression training \citep{rehman2025fort}. These methods share a common paradigm: parameterizing continuous neural transport maps optimized to minimize Wasserstein distance, enabling few-step sampling. \SQL uses OT fundamentally differently. Rather than learning continuous trajectories, we use OT-based multivariate quantiles \citep{hallin2021distribution,chernozhukov2017monge,hallin2022measure,ghosal2022multivariate} to define a fixed lattice that partitions the latent space, then solve discrete assignment problems to match data to quantiles.

\noindent\textbf{DGM for low sample regimes.}
Existing approaches for low-sample generation, ranging from data augmentation to specialized architectures like FastGAN~\citep{liu2021towards}, generally remain within paradigms that require auxiliary networks to process high-dimensional data. Consequently, these methods face severe sample complexity challenges when the sample size is small relative to the data dimension ($n \ll p$). \textsc{NeuroSQL} addresses this bottleneck by eliminating encoders and discriminators, thereby sidestepping the difficulty of learning functions from high-dimensional observation spaces.

\section{Conclusions}  
\label{sec. conclusions}

In this paper, we introduced NeuroSQL, a deep generative model that replaces stochastic encoders with rank-based quantile assignment. By learning embeddings through assignment rather than amortised inference, \SQL eliminates the posterior collapse typical of VAEs and avoids the adversarial dynamics of GANs. Furthermore, the deterministic assignment mechanism yields distributions substantially more interpretable than those of GANs, diffusion models, or VAEs. Unlike diffusion models, which require extensive denoising and large datasets, \SQL demonstrates superior performance in constrained settings (under $10^5$ samples) by avoiding high-variance score estimation. Evaluated on MNIST, CelebA, AFHQ, and OASIS, the model proves consistently competitive and often superior under matched experimental settings.

The primary contribution of \SQL is the introduction of a new principle for developing DGMs, prioritizing methodological innovation over high-capacity architecture. While our experiments validated the paradigm under controlled conditions, future work can evaluate its scalability to higher resolutions and comparisons against baselines like StyleGAN3. In terms of applications, particularly in medical data science, a natural next step is to experiment with the MedMNIST \citep{Yang_2023} benchmark suite. Moving forward, we foresee two key directions for NeuroSQL: (i) scaling the assignment mechanism to new modalities (e.g., audio, 3D), and (ii) expanding the theory of quantile-assignment training beyond its current empirical scope.

\newpage

\bibliography{iclr2026/sql_jasa,iclr2026/neural_sql,iclr2026/related_work}
\bibliographystyle{iclr2026_conference}

\newpage
\appendix
\onecolumn
\section{Acknowledgments}
We acknowledge the use of large language models for grammar checking, punctuation correction, spelling verification, and synonym suggestions to enhance writing clarity.

\section{Ethical and domain-specific notes (OASIS)}
\label{subsec:ethics}

In order to keep the experimental setup free of data leakage within the train/val/test splits, we performed subject-stratified cross-validation. We highlight that one must not equate synthetic imaging data with clinical data. Synthetic imaging data, however, may have practical utilities, such as for treating missing data, but this is beyond the scope of this paper. Here, we generate synthetic data to demonstrate the efficacy of \SQL as a generative model; we use the generated imaging data to evaluate the model's performance; we do not claim its clinical utility. Future work should verify this independently, and we will release seeds, splits, and scripts to facilitate further validations.

\vspace{-2mm}
\section{Datasets}
\label{subsec:datasets}

\noindent\textbf{MNIST \citep{lecun1998gradient}.} Handwritten digits ($60\mathrm{k}$ train, $10\mathrm{k}$ test, $28{\times}28$ grayscale). We replicate to RGB for evaluation only.

\noindent\textbf{CelebA \citep{liu2015faceattributes}.} Face dataset ($202{,}599$ images in full) of which we use around 2500 by preprocessing them. We use them center-cropped and resized to $128{\times}128$. Training is unconditional despite available attributes.

\noindent\textbf{AFHQ \citep{choi2020stargan}.} Animal Faces HQ contains $15{,}000$ high-quality animal face images across cats, dogs, and wildlife at $512{\times}512$ resolution originally. Within experiments, image size was reduced to $128{\times}128$ and total number of images used was about $2000$.

\noindent\textbf{OASIS \citep{marcus2007open,ninadaithal2023imagesoasis}.} Neuroimaging dataset with $\sim 80\mathrm{k}$ MRI slices from $416$ subjects, downsampled to $128{\times}128$. Medical images test resistance to overfitting in constrained domains. We used the version of OASIS that is publicly available on Kaggle containing a chunk of the whole dataset. It is pre-structured and is organized in four subfolders, namely: Mild Dementia, Moderate Dementia, Non Demented, and Very Mild Demented. We use it for the fact that it is publicly available.
\vspace{-2mm}

\section{Implementation details}
\label{app:implementation details}

\noindent\textbf{Generator backbones.}
Within each experimental domain, NeuroSQL, VAE, and GAN share the \emph{\textbf{identical}} generator architecture and initialization.
The following generators were used.

\textbf{ConvNet (Baseline)} is a parameter-efficient standard deconvolutional network (similar to DCGAN). It consists of a linear projection followed by a stack of transposed convolution layers with halving channel widths ($512{\to}256{\to}128{\to}\dots$) and batch normalization. \textbf{ResNet:} A high-capacity generator comprising a linear projection followed by four residual upsampling blocks. Unlike the baseline, this backbone maintains a \emph{constant channel width} of 512 throughout all blocks. \textbf{U-Net:} A conditional architecture where the latent vector $z$ modulates a learnable constant input tensor via Feature-wise Linear Modulation (FiLM) layers injected at every resolution level. To keep the evaluation pipeline of the work fair, for diffusion we employ a standard U-Net, scaling the channel width so that the total trainable parameter count remains within ${\pm}10\%$ of the competing generator.

\noindent\textbf{Budget parity.}
All methods are trained under the same compute budget per dataset: identical number of epochs, batch size, optimizer (AdamW), learning-rate schedule (cosine decay), gradient clipping, and early stopping criteria.

\noindent\textbf{Latent dimensionality sweep.}
We vary the latent dimension $q \in \{2,\ldots,128\}$ to assess sensitivity to representation size and to verify that trends are not driven by a particular latent choice. Although we conducted experiments with a latent dimension of 128, our experiments suggest that such high dimensionality is unnecessary to capture the intrinsic geometry of the datasets studied.

\noindent\textbf{Resolution ablation.}
On OASIS, we evaluate two spatial resolutions, $64\times 64$ and $128\times 128$, to probe robustness to increased output dimensionality.

\noindent\textbf{Loss sensitivity for assignment.}
We ablate the cost used to form the assignment matrix, comparing a pixelwise $\ell_2$ loss against our default perceptual cost
$\ell = \tfrac{1}{2}(1-\mathrm{SSIM}) + \tfrac{1}{2}\|\cdot\|_1$.
Unless stated otherwise, the main results use the SSIM+$\ell_1$ cost (additional details in the Appendix).

\noindent\textbf{Training of NeuroSQL.}
For \SQL we used the Sobol or Uniform lattice in $[0,1]^d$ to construct the quantiles.
We performed the \emph{Hungarian} or \emph{Greedy} algorithm at every $K$ epochs (with $K\!\in\!\{2,3,5\}$ treated as a hyperparameter), and selected the momentum parameter as $\rho{=}0.7$. 

\noindent\textbf{Additional details. }
To ensure computational parity, we standardize the optimization protocol across all methods using AdamW, cosine decay, gradient clipping, and early stopping. Full architectural details and hyperparameters are provided in Appendix~\ref{app:protocol}

\vspace{-2mm}
\section{Evaluation metrics}
\label{subsec:metrics}
\noindent
\textbf{FID (proxy; $\downarrow$)}:
Because our image domains differ from ImageNet/COCO, we report a \emph{proxy-FID} as a coarse indicator only. Specifically, we extract features using a \emph{fixed} lightweight convnet (random initialization with a fixed seed; 128-d features) and compute the Fr\'echet distance between the resulting feature distributions of real and generated samples. We match the sample counts ($n_{\text{real}}=n_{\text{gen}}=2048$ by default) and fix the metric sampling seed for reproducibility.
\textbf{LPIPS (↓)}: VGG backbone via LPIPS (fallback: cosine similarity on VGG features); mean over 50 paired real–fake images. \;
\textbf{SSIM (↑)}: mean over the same 50 pairs. \;
Unless noted, we evaluate \SQL in \emph{sampled-$z$} mode ($z\!\sim\!\mathcal{N}(0,I)$); \emph{paired-$z$} (reconstruction) results appear in the supplement.

\section{Mini-batch Training NeuroSQL}

While Algorithm \ref{alg:neurosql} describes the exact full-batch procedure, \SQL scales to large datasets ($n \gg 10^4$) via a stochastic approximation outlined in Algorithm \ref{alg:neurosql_stochastic}. In this regime, we maintain a persistent "memory bank" of latent codes $\widehat{\bm Z}$. In each iteration, we sample a mini-batch of data $\bm X_{\mathcal{B}}$ and a corresponding random subset of the lattice $\bm Q_{\mathcal{K}}$. The assignment problem is solved locally on the $m \times m$ cost matrix. This reduces the computational complexity of the assignment step from $O(n^3)$ to $O((n/m) \cdot m^3) = O(n \cdot m^2)$ per epoch using the Hungarian method, or $O(n \cdot m^2)$ using the Greedy alternative. The momentum parameter $\rho$ is critical here, acting as a temporal smoother that stabilizes the stochastic assignment trajectory towards the optimal transport map.

This computational benefit is reflected in Fig.~\ref{fig:scalability}, which reports mean epoch time as a function of sample size (N): \SQL scales comparably to a VAE baseline and substantially more favorably than a GAN as (N) grows. Moreover, increasing the local assignment size (m) increases runtime in line with the $O(n \cdot m^2)$ per-epoch dependence predicted by the mini-batch scheme.

\begin{figure}[h]
\centering
\includegraphics[width=0.45\linewidth]{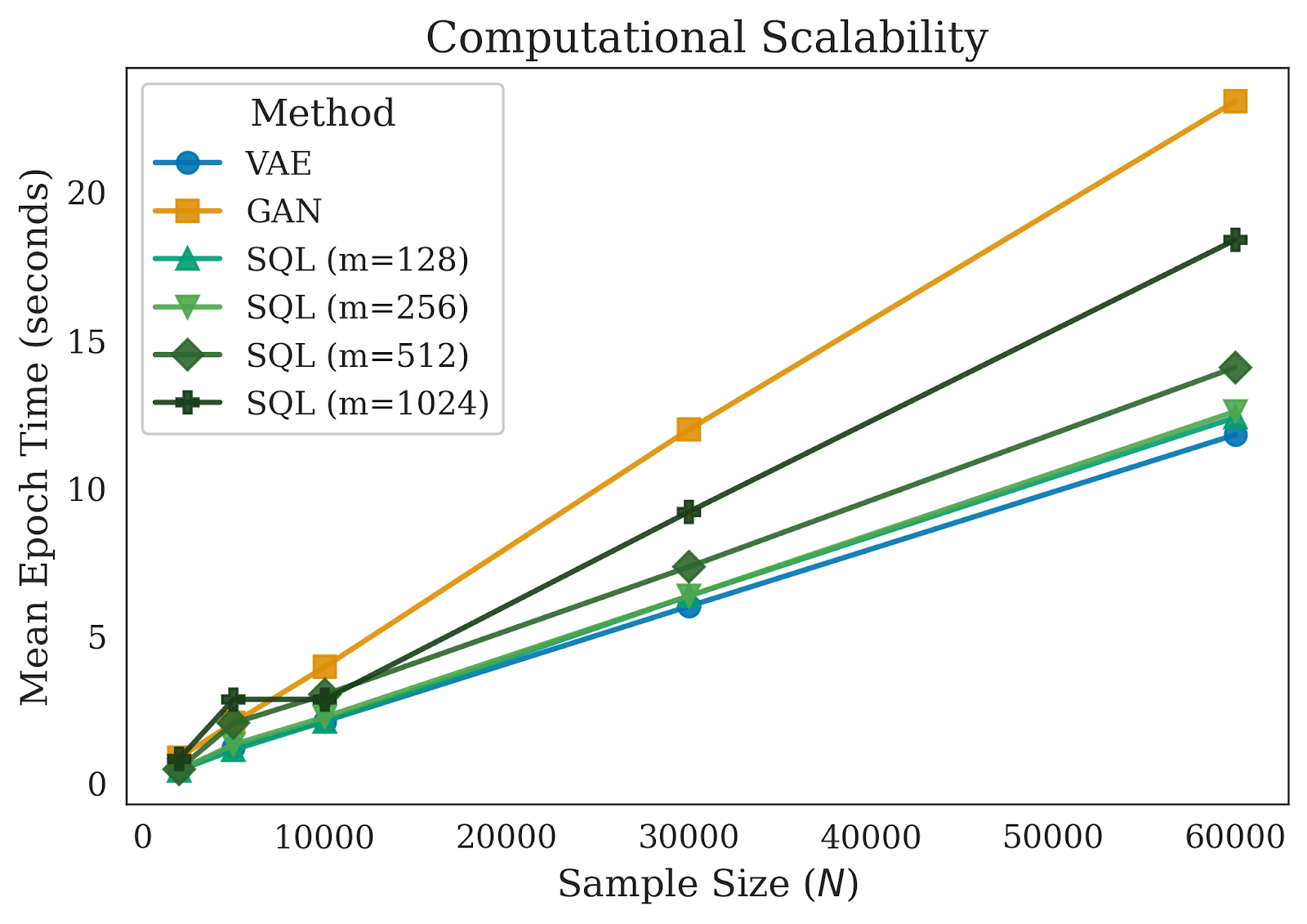}
\caption{Computational scalability: mean epoch time versus sample size $N$ for VAE, GAN, and \SQL with different local assignment sizes $m$.}
\label{fig:scalability}
\end{figure}

\begin{algorithm}[tb]
\caption{Quantile Assignment in Mini-batch Setting (Large Scale)}
\label{alg:neurosql_stochastic}
\begin{algorithmic}[1]
\State \textbf{Input:} data $\{\bm X_i\}_{i=1}^n$, prior $P_Z$, lattice $\bm Q^n$,
batch size $m$, epochs $E$, momentum $\rho\in[0,1)$, step size $\eta$
\State \textbf{Initialize:} global latent codes $\widehat{\bm Z}=\{\widehat{\bm Z}_1,\dots,\widehat{\bm Z}_n\}$
by random assignment from $\bm Q^n$
\For{$e=1,\dots,E$}
  \State Shuffle indices $\{1,\dots,n\}$
  \For{each mini-batch $\mathcal{B}\subset\{1,\dots,n\}$ with $|\mathcal{B}|=m$}
    \State \textbf{1) Generator update:}
    \State $\bm\theta \leftarrow \bm\theta
      - \eta \nabla_{\bm\theta}\left[
      \frac{1}{m}\sum_{i\in\mathcal{B}}
      \ell\!\big(\bm X_i,\bm G_{\bm\theta}(\widehat{\bm Z}_i)\big)\right]$
    \State \textbf{2) Stochastic quantile assignment:}
    \State Sample lattice subset indices $\mathcal{K}\subset\{1,\dots,n\}$ with $|\mathcal{K}|=m$
    \State Form batch cost matrix $\bm C\in\mathbb{R}^{m\times m}$ with entries
    \State \hspace{1em}$C_{j,k}\leftarrow
      \ell\!\big(\bm X_{\mathcal{B}[j]},\bm G_{\bm\theta}(\bm Q_{\mathcal{K}[k]})\big)$,
      for $j,k\in\{1,\dots,m\}$
    \State Solve batch assignment $\pi_{\text{batch}}$ on $\bm C$ (Hungarian or greedy)
    \State \textbf{3) Momentum update (sparse):}
    \For{$j=1,\dots,m$}
      \State $\widehat{\bm Z}_{\mathcal{B}[j]} \leftarrow
        \rho\,\bm Q_{\mathcal{K}[\pi_{\text{batch}}(j)]}
        + (1-\rho)\,\widehat{\bm Z}_{\mathcal{B}[j]}$
    \EndFor
  \EndFor
\EndFor
\State \textbf{Output:} generator $\bm G_{\bm\theta}$, aligned latents $\widehat{\bm Z}$
\end{algorithmic}
\end{algorithm}

\section{Greedy Assignment Algorithm}\label{appendix:greedy}

To address the scalability limitations of the Hungarian algorithm (which scales as $\mathcal{O}(n^3)$) for larger batch sizes or datasets, we implement a Greedy Assignment strategy. While the Hungarian algorithm guarantees the global minimum cost for the linear assignment problem, the Greedy approach provides an approximation that is computationally efficient ($\mathcal{O}(n^2)$) and sufficient for maintaining training stability in the paradigm we introduce with our model.

\noindent\textbf{Algorithm Description.}
The greedy strategy iterates over each row of the cost matrix. For each row $i$, it selects the column $j$ that minimizes the cost $C_{i,j}$, provided that column $j$ has not already been assigned to a previous row. Once a column is selected, it is removed from the pool of available columns.

\begin{algorithm}[tb]
\caption{Greedy Assignment ($\mathcal{O}(n^2)$)}
\label{alg:greedy}
\begin{algorithmic}[1]
\State \textbf{Input:} cost matrix $\bm C \in \mathbb{R}^{n \times n}$
\State \textbf{Initialize:} unassigned columns $\mathcal{S} \leftarrow \emptyset$, permutation $\pi \in \{1,\dots,n\}^{n}$
\For{$i=1,\dots,n$}
  \State $c_{\min} \leftarrow +\infty$, $j^\star \leftarrow 0$
  \For{$j=1,\dots,n$}
    \If{$j \notin \mathcal{S} \;\land\; C_{i,j} < c_{\min}$}

      \State $c_{\min} \leftarrow C_{i,j}$
      \State $j^\star \leftarrow j$
    \EndIf
  \EndFor
  \State $\pi(i) \leftarrow j^\star$
  \State $\mathcal{S} \leftarrow \mathcal{S}\cup\{j^\star\}$
\EndFor
\State \textbf{Output:} assignment $\pi$
\end{algorithmic}
\end{algorithm}

\noindent\textbf{Computational Complexity.}
The outer loop runs exactly $n$ times (once for each data sample). The inner loop scans $n$ columns to find the minimum unassigned cost. Although the number of available columns decreases by 1 in each iteration, the upper bound of the search remains $n$. Consequently, the total complexity is proportional to $\sum_{i=1}^{n} n = n^2$, yielding a time complexity of $\mathcal{O}(n^2)$.

With Table \ref{tab:ablation_unet} we try to show the quantitative difference that occurs when the assignment algorithm changes.

\begin{table}[!t]
\centering
\caption{Ablation Study: U-Net Latent Dimensions}
\begin{tabular}{clccc}
\hline
\textbf{Latent Dim} & \textbf{Method} & \textbf{FID (proxy)} $\downarrow$ & \textbf{LPIPS} $\downarrow$ & \textbf{SSIM} $\uparrow$ \\
\hline
8 & Greedy & 10.764394 & 0.376178 & 0.296806 \\
 & Hungarian & 11.121147 & 0.378670 & 0.274640 \\
\hline
16 & Greedy & 11.632382 & 0.398841 & 0.241317 \\
 & Hungarian & 11.798410 & 0.390295 & 0.240504 \\
\hline
24 & Greedy & 13.222514 & 0.384961 & 0.272349 \\
 & Hungarian & 11.338643 & 0.379329 & 0.271508 \\
\hline
\end{tabular}
\label{tab:ablation_unet}
\end{table}

\newpage

\newpage
\section{Quantitative evaluation of Runtime and Memory}
Here, we would like to highlight and provide a quantitative evaluation of the memory and runtime of our proposed model. The following Table \ref{tab:mnist-ablation-memory-runtime} is based on the metrics gathered from an MNIST run.

\begin{table}[ht]
\vspace{0.1in}
\centering
\caption{Ablation: Assignment Methods and Resource Usage on MNIST (seed=11). Metrics averaged over epochs where applicable.}
\label{tab:mnist-ablation-memory-runtime}
\scriptsize
\setlength{\tabcolsep}{2.5pt}
\renewcommand{\arraystretch}{1.05}

\resizebox{\columnwidth}{!}{%
\begin{tabular}{c l l c c c c c c c c}
\toprule
\shortstack{Latent\\Dim} &
\shortstack{Assign.\\Method} &
Method &
\shortstack{Peak\\VRAM\\(MB)} &
\shortstack{RAM\\Used\\(MB)} &
\shortstack{Mean\\Epoch\\Time (s)} &
\shortstack{Mean\\Epoch\\GPU\\(MB)} &
\shortstack{Mean\\Epoch\\$\Delta$RAM\\(MB)} &
\shortstack{FID\\(proxy)$\downarrow$} &
\shortstack{LPIPS$\downarrow$} &
\shortstack{SSIM$\uparrow$} \\
\midrule
\multirow{4}{*}{2}
& greedy    & NeuroSQL & 81.5   & 2764 & 1.40 & 116 & 11.2 & 0.628 & 0.045 & 0.580 \\
& hungarian & NeuroSQL & 284.6  & 3078 & 1.37 & 293 & 0.39 & 0.576 & 0.041 & 0.616 \\
&          & VAE      & 319.8  & 3091 & 1.33 & 320 & 0.10 & 1.259 & 0.050 & 0.249 \\
&          & GAN      & 405.6  & 3115 & 1.56 & 406 & 0.17 & 2.060 & 0.064 & 0.223 \\
\midrule
\multirow{4}{*}{3}
& greedy    & NeuroSQL & 284.8  & 3206 & 1.39 & 293 & 0.14 & 0.567 & 0.040 & 0.574 \\
& hungarian & NeuroSQL & 284.9  & 3298 & 1.38 & 293 & 0.00 & 0.556 & 0.037 & 0.607 \\
&          & VAE      & 320.7  & 3284 & 1.34 & 321 & 0.10 & 1.544 & 0.053 & 0.185 \\
&          & GAN      & 404.6  & 3293 & 1.56 & 404 & 0.01 & 5.443 & 0.327 & 0.070 \\
\bottomrule
\end{tabular}%
}
\vspace{0.1in}
\end{table}

\section{Models and training protocol}
\label{app:protocol}

\noindent\textbf{Common setup.}
Images are scaled to $[0,1]$ (diffusion uses $[-1,1]$ internally). We use AdamW, cosine annealing, gradient clipping, and early stopping on validation loss. Generator backbones are matched across methods for capacity parity.

\noindent\textbf{NeuroSQL (ours).}
We construct a size-$n$ deterministic latent lattice via scrambled Sobol points mapped coordinatewise through $F^{-1}_{Z_\ell}$ (Sobol$\to$Gaussian). Every $K$ epochs we solve an exact global assignment (Hungarian) between data and lattice codes, where $K \in \{2,3,5\}$ is selected as a hyperparameter. After each assignment, we apply latent momentum
$\widehat{z}^{(t)} \leftarrow \rho \, z_{\pi^{(t)}(i)} + (1-\rho) \, \widehat{z}^{(t-1)}$ with $\rho = 0.7$.
The decoder is trained by regression on assigned codes using
$\ell = \frac{1}{2}(1-\mathrm{SSIM}) + \frac{1}{2}\|\cdot\|_1$.

\noindent\textbf{VAE.}
We reuse the same generator backbone as the decoder; the encoder is an MLP on flattened pixels (to keep capacity modest). Training uses SSIM$+$L1 reconstruction plus a $\beta$-scaled KL term with $\beta = 0.005$.

\noindent\textbf{GAN.}
The generator backbone is identical to NeuroSQL's. The discriminator is a lightweight four-layer CNN. We use the non-saturating objective with BCE logits, sharing optimizer, scheduler, and gradient clipping with NeuroSQL.

\noindent\textbf{Diffusion (DDPM).}
A compact U-Net (base width $32$) is trained with a linear $\beta$ schedule for $T = 1000$ steps; default sampling uses $100$ steps to match compute. Inputs are normalized to $[-1,1]$ following common practice.

\noindent\textbf{Reproducibility knobs.}
We fix random seeds, match the number of training epochs and batch sizes across methods, and report all per-method hyperparameters (including learning rates, $K \in \{2,3,5\}$, and augmentations).

\section{Preprocessing and splits} 
\label{app:preproc}

For each dataset, we standardize a lightweight pipeline to minimize confounds while allowing small variations across runs.

\begin{itemize}
\item \textbf{Resize/crop.} \texttt{MNIST}: native $28{\times}28$. \texttt{CelebA}: center-crop then resize, typically to $32{\times}32$ (primary), with occasional $64$–$128$ experiments. \texttt{OASIS}: center-crop then resize, primarily $128{\times}128$ with $64{\times}64$ ablations.
\item \textbf{Scaling.} Inputs mapped to $[0,1]$ (no per-image standardization during training).
\item \textbf{Channel handling.} \texttt{MNIST}/\texttt{OASIS} trained as single-channel; for backends expecting 3 channels (e.g., LPIPS/VGG), we replicate channels \emph{at metric time only}.
\item \textbf{Splits.} \texttt{MNIST}: standard train/test. \texttt{CelebA}: official train/val/test. \texttt{OASIS}: subject-wise 80/10/10 to prevent slice leakage. Seeds and split indices are provided in the supplementary.
\end{itemize}

\noindent\textbf{Optimization and budgets (typicals/ranges).}
AdamW (or Adam), cosine warm restarts; weight decay $\sim 10^{-4}$; gradient clip $\approx 1.0$; early stopping on validation loss with patience $\sim 25$–$30$ epochs. Learning rates are usually in $[1\times 10^{-4},\,3\times 10^{-4}]$ for convolutional decoders; diffusion runs use comparable schedules at matched compute. Batch sizes depend on resolution: \texttt{MNIST} $32$–$64$, \texttt{CelebA} $64$–$128$, \texttt{OASIS} $32$–$64$. Epoch caps are typically $120$–$250$ across datasets, with early stopping often terminating earlier. We sweep latent dimension $d$ over $\{2,4,8,16,32,64,128\}$ and report results (Sec.~\ref{tab:results}).

\section{Proof of Proposition \ref{prop.}}
\label{app:proof}

\begin{proof}
Following \cite{hallin2021distribution}, we define the center-outward empirical distribution function as the solution of the optimal transportation problem:
$$ F^n_\pm = \argmin_{T\in \mathcal{T}} \sum_{i=1}^n \Vert \bm Z_i - T(\bm Z_i) \Vert^2, $$ 
where the minimum is taken over $\mathcal{T}$, the set of all bijective mappings between $\bm Z_1, ..., \bm Z_n$ and the grid $\bm U_1, \dots, \bm U_n$.
In fact, this is equivalent to solve a linear assignment problem:
$$ \pi^* = \argmin_{\pi \in S_n} \Vert \bm Z - \bm U^n_\pi \Vert^2,$$ 
and set $F^n_\pm(\bm Z^n) := \bm U^n_{\pi^*} $.
Then we apply Theorem 2.4 
in \cite{hallin2021distribution} to obtain:
$$\max_{1\leq i\leq n} \Vert F^n_\pm(\bm Z_i) - F_\pm(\bm Z_i) \Vert^2 \rightarrow 0 \text{ a.s. as } n \rightarrow \infty.$$ 
In our case, as $F_\pm$ is the uniform distribution over $\mathcal{U}_d$,
it is relatively straightforward to see that: 
$$ \min_{\pi \in S_n} \Vert \bm Q^n_\pi -\bm Z \Vert^2 \rightarrow 0, \text{ a.s. as } n \rightarrow \infty.$$
\end{proof}

\section{Practical training details}
\label{subsec:training}

\noindent\textbf{Loss choices.} For images, we use a perceptual, scale-stable loss:
\[
\ell(\hat{x},x) = \frac{1}{2}\big(1-\mathrm{SSIM}(\hat{x},x)\big) + \frac{1}{2}\|\hat{x}-x\|_1,
\]
which is the exact loss used in our codebase.

We instantiate $\mathrm{gen}$ with lightweight decoders so that comparisons against VAEs/GANs/Diffusion control for capacity and compute:
\begin{itemize}
\item \textbf{ConvNet.} Transposed-convolution stack mapping $z \in \mathbb{R}^q$ to $X \in \mathbb{R}^{3 \times H \times H}$ (stride-2 upsampling).
\item \textbf{ResNet} Four residual upsampling blocks (512$\to$256$\to$128$\to$64), followed by a $3 \times 3$ head with sigmoid output in $[0,1]$. We optionally initialize residual weights from ResNet-18 where shapes match.
\item \textbf{U-Net decoder.} A small transformer decoder on patchified embeddings of $z$ followed by an MLP head back to pixels.
\end{itemize}
Our experiments keep these decoders small and matched across methods to stress that gains come from the \emph{quantile--assignment loop}, not decoder sophistication.

\begin{itemize}
\item \textbf{Loss and normalization.} Images are scaled to $[0,1]$. We use $\ell = \frac{1}{2}(1-\mathrm{SSIM}) + \frac{1}{2}\ell_1$ in both decoder and cost matrix. 
\item \textbf{Optimization.} AdamW with cosine annealing and gradient clipping; early stopping on validation $\ell$.
\item \textbf{Latent momentum.} After each assignment, a momentum update $\widehat{z}^{(t)} \leftarrow \rho \, z_{\pi^{(t)}(i)} + (1-\rho) \, \widehat{z}^{(t-1)}$ stabilizes training (we use $\rho = 0.7$).
\item \textbf{Resource parity.} For fair comparisons to VAEs, GANs, and Diffusion, we fix the \emph{same} generator backbone and training budget; only the learning paradigm changes.
\end{itemize}

\section{Additional Results on OASIS, CelebA, AFHQ and MNIST}\label{tab:results}

In this section, we present a comprehensive evaluation of NeuroSQL on the OASIS brain imaging dataset, the CelebA dataset, and the AFHQ dataset, using the different distinct generator backbones: ConvNet, ResNet, and U-Net. We benchmark performance against VAE and GAN baselines using the same generators across a wide range of latent dimensions to assess model stability.

\begin{table}[h]
\centering
\caption{Results on OASIS, a brain imaging dataset, using ConvNet across latent dimensions. Lower is better for FID (proxy) and LPIPS; higher is better for SSIM.}
\label{tab:oasis-convnet}
\small
\vspace{2mm}
\begin{tabular}{@{}r l r r r@{}}
\toprule
\textbf{Latent dimension} & \textbf{Method} & \textbf{FID (Proxy) $\downarrow$} & \textbf{LPIPS $\downarrow$} & \textbf{SSIM $\uparrow$} \\
\midrule
\multirow{3}{*}{2}   & NeuroSQL & 7.914  & 0.390 & 0.259 \\
                     & VAE & 8.198  & 0.410 & 0.246 \\
                     & GAN & 9.936  & 0.450 & 0.218 \\
\midrule
\multirow{3}{*}{4}   & NeuroSQL & 9.241  & 0.403 & 0.257 \\
                     & VAE & 8.720  & 0.463 & 0.212 \\
                     & GAN & 19.420 & 0.509 & 0.206 \\
\midrule
\multirow{3}{*}{8}   & NeuroSQL & 8.286  & 0.411 & 0.256 \\
                     & VAE & 8.029  & 0.485 & 0.215 \\
                     & GAN & 12.766 & 0.558 & 0.193 \\
\midrule
\multirow{3}{*}{16}  & NeuroSQL & 8.602  & 0.455 & 0.267 \\
                     & VAE & 12.768 & 0.539 & 0.171 \\
                     & GAN & 12.485 & 0.584 & 0.200 \\
\midrule
\multirow{3}{*}{32}  & NeuroSQL & 8.856  & 0.401 & 0.257 \\
                     & VAE & 12.490 & 0.525 & 0.178 \\
                     & GAN & 14.852 & 0.593 & 0.174 \\
\midrule
\multirow{3}{*}{64}  & NeuroSQL & 7.385  & 0.388 & 0.265 \\
                     & VAE & 16.003 & 0.567 & 0.151 \\
                     & GAN & 14.453 & 0.602 & 0.125 \\
\midrule
\multirow{3}{*}{128} & NeuroSQL & 18.743 & 0.453 & 0.248 \\
                     & VAE & 16.329 & 0.571 & 0.158 \\
                     & GAN & 29.792 & 0.620 & 0.084 \\
\bottomrule
\end{tabular}
\end{table}

\begin{table}[h]
\centering
\caption{Results on OASIS, a brain imaging dataset, using ResNet across latent dimensions. Lower is better for FID (proxy) and LPIPS; higher is better for SSIM.}
\label{tab:oasis-resnet}
\small
\vspace{2mm}
\begin{tabular}{@{}r l r r r@{}}
\toprule
\textbf{Latent dimension} & \textbf{Method} & \textbf{FID (proxy) $\downarrow$} & \textbf{LPIPS $\downarrow$} & \textbf{SSIM} \\
\midrule
\multirow{3}{*}{16}  & NeuroSQL & 25.410  & 0.249 & 0.301 \\
                     & VAE & 51.139  & 0.309 & 0.171 \\
                     & GAN & 168.993 & 0.727 & 0.008 \\
\midrule
\multirow{3}{*}{32}  & NeuroSQL & 31.957  & 0.257 & 0.242 \\
                     & VAE & 66.045  & 0.362 & 0.128 \\
                     & GAN & 158.103 & 0.687 & 0.008 \\
\midrule
\multirow{3}{*}{64}  & NeuroSQL & 30.892  & 0.220 & 0.221 \\
                     & VAE & 47.146  & 0.358 & 0.135 \\
                     & GAN & 156.088 & 0.741 & 0.135 \\
\midrule
\multirow{3}{*}{128} & NeuroSQL & 34.346  & 0.224 & 0.196 \\
                     & VAE & 52.317  & 0.397 & 0.135 \\
                     & GAN & 156.288 & 0.723 & 0.124 \\
\bottomrule
\end{tabular}
\end{table}

\begin{table}[h]
\centering
\caption{Results on OASIS, a brain imaging dataset, using a U-Net across latent dimensions. Lower is better for FID (proxy) and LPIPS; higher is better for SSIM.}
\label{tab:oasis-U-Net-full}
\small
\vspace{2mm}
\begin{tabular}{@{}r l r r r@{}}
\toprule
\textbf{Latent dimension} & \textbf{Method} & \textbf{FID (proxy) $\downarrow$} & \textbf{LPIPS $\downarrow$} & \textbf{SSIM $\uparrow$} \\
\midrule
\multirow{3}{*}{4}   & NeuroSQL & 8.519902  & 0.386861 & 0.254752 \\
                     & VAE      & 7.858056  & 0.377805 & 0.265082 \\
                     & GAN      & 26.234322 & 0.597633 & 0.206676 \\
\midrule
\multirow{3}{*}{8}   & NeuroSQL & 8.579974  & 0.401890 & 0.238080 \\
                     & VAE      & 5.976874  & 0.384821 & 0.264277 \\
                     & GAN      & 35.444614 & 0.609803 & 0.172225 \\
\midrule
\multirow{3}{*}{16}  & NeuroSQL & 10.037155 & 0.391049 & 0.260999 \\
                     & VAE      & 7.178138  & 0.402099 & 0.231028 \\
                     & GAN      & 20.907921 & 0.616312 & 0.225444 \\
\midrule
\multirow{3}{*}{32}  & NeuroSQL & 8.405630  & 0.388190 & 0.253740 \\
                     & VAE      & 4.424680  & 0.402910 & 0.238710 \\
                     & GAN      & 28.069950 & 0.613550 & 0.200620 \\
\midrule
\multirow{3}{*}{64}  & NeuroSQL & 8.259449  & 0.385646 & 0.262159 \\
                     & VAE      & 6.098939  & 0.395427 & 0.272117 \\
                     & GAN      & 30.385052 & 0.644507 & 0.153054 \\
\midrule
\multirow{3}{*}{128} & NeuroSQL & 7.553965  & 0.371082 & 0.282864 \\
                     & VAE      & 9.101107  & 0.415775 & 0.257312 \\
                     & GAN      & 30.505713 & 0.553681 & 0.202557 \\
\bottomrule
\end{tabular}
\end{table}

\begin{table}[h]
\begin{center}
\caption{Results on OASIS, a brain imaging dataset, using a U-Net — with results averaged across latent dimensions $\{4,8,16,32,64,128\}$.}
\label{tab:oasis-U-Net-method-avg}
\small
\vspace{2mm}
\begin{tabular}{@{}l r r r@{}}
\toprule
\textbf{Method} & \textbf{FID (proxy) $\downarrow$} & \textbf{LPIPS $\downarrow$} & \textbf{SSIM $\uparrow$} \\
\midrule
NeuroSQL & 8.559346  & \textbf{0.387453} & \textbf{0.258766} \\
VAE & \textbf{6.772966} & 0.396473 & 0.254754 \\
GAN & 28.591262 & 0.605914 & 0.193429 \\
\bottomrule
\end{tabular}
\end{center}
\break
\footnotesize On average, across latent dimensions, NeuroSQL attains the best LPIPS and SSIM, while VAE has the lowest FID (proxy).
\end{table}


\begin{table}[h]
\centering
\caption{Results on CelebA, a face attributes dataset, using ConvNet across latent dimensions. Lower is better for FID (proxy) and LPIPS; higher is better for SSIM.}
\label{tab:celeba-convnet-ablation}
\small
\vspace{2mm}
\begin{tabular}{@{}r l r r r@{}}
\toprule
\textbf{Latent dimension} & \textbf{Method} & \textbf{FID (proxy) $\downarrow$} & \textbf{LPIPS $\downarrow$} & \textbf{SSIM $\uparrow$} \\
\midrule
\multirow{3}{*}{2}   & NeuroSQL & 9.64453  & 0.22094 & 0.31376 \\
                     & VAE & 11.49087 & 0.24068 & 0.27396 \\
                     & GAN & 26.56548 & 0.47277 & 0.10216 \\
\midrule
\multirow{3}{*}{4}   & NeuroSQL & 6.99789  & 0.22845 & 0.31645 \\
                     & VAE & 8.57467  & 0.23091 & 0.26471 \\
                     & GAN & 29.68144 & 0.49874 & 0.06955 \\
\midrule
\multirow{3}{*}{8}   & NeuroSQL & 6.686607 & 0.244427 & 0.297947 \\
                     & VAE & 32.183292& 0.281024 & 0.226046 \\
                     & GAN & 28.772638& 0.461147 & 0.058802 \\
\midrule
\multirow{3}{*}{16}  & NeuroSQL & 17.252127& 0.304758 & 0.296877 \\
                     & VAE & 11.955338& 0.282781 & 0.160219 \\
                     & GAN & 20.922581& 0.419844 & 0.117695 \\
\midrule
\multirow{3}{*}{32}  & NeuroSQL & 4.04047  & 0.19269  & 0.25707 \\
                     & VAE & 6.57785  & 0.21120  & 0.13975 \\
                     & GAN & 13.92531 & 0.31649  & 0.12201 \\
\midrule
\multirow{3}{*}{64}  & NeuroSQL & 3.94787  & 0.20649  & 0.25782 \\
                     & VAE & 3.93301  & 0.21478  & 0.15485 \\
                     & GAN & 16.27481 & 0.35050  & 0.10456 \\
\midrule
\multirow{3}{*}{128} & NeuroSQL & 4.91280  & 0.20557  & 0.26119 \\
                     & VAE & 5.57425  & 0.20361  & 0.16837 \\
                     & GAN & 17.16048 & 0.34640  & 0.15798 \\
\bottomrule
\end{tabular}
\end{table}

\begin{table}[h]
\centering
\caption{Results on CelebA, a face attributes dataset, using ResNet across latent dimensions. Lower is better for FID (proxy) and LPIPS; higher is better for SSIM.}
\label{tab:celeba-resnet}
\small
\vspace{2mm}
\begin{tabular}{@{}r l r r r@{}}
\toprule
\textbf{Latent dimension} & \textbf{Method} & \textbf{FID (proxy) $\downarrow$} & \textbf{LPIPS $\downarrow$} & \textbf{SSIM $\uparrow$} \\
\midrule
\multirow{3}{*}{2}   & NeuroSQL & 4.40599  & 0.18367 & 0.27506 \\
                     & VAE & 8.01552  & 0.19720 & 0.24972 \\
                     & GAN & 19.01699 & 0.20240 & 0.13535 \\
\midrule
\multirow{3}{*}{4}   & NeuroSQL & 4.47511  & 0.16968 & 0.26754 \\
                     & VAE & 6.03891  & 0.20652 & 0.20760 \\
                     & GAN & 14.21587 & 0.26898 & 0.14515 \\
\midrule
\multirow{3}{*}{8}   & NeuroSQL & 4.60623  & 0.18201 & 0.25536 \\
                     & VAE & 4.86620  & 0.25218 & 0.17823 \\
                     & GAN & 13.94172 & 0.18593 & 0.17628 \\
\midrule
\multirow{3}{*}{32}  & NeuroSQL & 3.99240  & 0.18115 & 0.21823 \\
                     & VAE & 6.00124  & 0.25017 & 0.12915 \\
                     & GAN & 11.08812 & 0.23644 & 0.16986 \\
\bottomrule
\end{tabular}
\end{table}


\begin{table}[h]
\centering
\caption{Results on CelebA, a face attributes dataset, using U-Net (unconditional) across latent dimensions. Lower is better for FID (proxy) and LPIPS; higher is better for SSIM.}
\label{tab:celeba-U-Net-full}
\small
\vspace{2mm}
\begin{tabular}{@{}r l r r r@{}}
\toprule
\textbf{Latent dimension} & \textbf{Method} & \textbf{FID (proxy) $\downarrow$} & \textbf{LPIPS $\downarrow$} & \textbf{SSIM $\uparrow$} \\
\midrule
\multirow{3}{*}{2}  & NeuroSQL & 4.96169  & 0.18253 & 0.27537 \\
                    & VAE      & 7.33227  & 0.19269 & 0.25696 \\
                    & GAN      & 14.09393 & 0.22041 & 0.12822 \\
\midrule
\multirow{3}{*}{4}  & NeuroSQL & 3.54758  & 0.18932 & 0.27993 \\
                    & VAE      & 4.99991  & 0.19061 & 0.22841 \\
                    & GAN      & 7.66398  & 0.19251 & 0.14899 \\
\midrule
\multirow{3}{*}{8}  & NeuroSQL & 3.96161  & 0.19244 & 0.24416 \\
                    & VAE      & 5.48428  & 0.20413 & 0.20573 \\
                    & GAN      & 20.46054 & 0.16577 & 0.18463 \\
\midrule
\multirow{3}{*}{16} & NeuroSQL & 2.72089  & 0.17970 & 0.24505 \\
                    & VAE      & 10.00917 & 0.19001 & 0.18914 \\
                    & GAN      & 14.48838 & 0.18500 & 0.15935 \\
\midrule
\multirow{3}{*}{32} & NeuroSQL & 14.65015 & 0.21850 & 0.27135 \\
                    & VAE      & 31.10957 & 0.24081 & 0.21405 \\
                    & GAN      & 30.37928 & 0.31335 & 0.13607 \\
\bottomrule
\end{tabular}
\end{table}


\begin{table}[h]
\centering
\caption{Results on AFHQ, an animal faces dataset, using ConvNet across latent dimensions. Lower is better for FID (proxy) and LPIPS; higher is better for SSIM.}
\label{tab:animals-conv256-full}
\small
\vspace{2mm}
\begin{tabular}{@{}r l r r r@{}}
\toprule
\textbf{Latent dimension} & \textbf{Method} & \textbf{FID (proxy) $\downarrow$} & \textbf{LPIPS $\downarrow$} & \textbf{SSIM $\uparrow$} \\
\midrule
\multirow{3}{*}{2}   & NeuroSQL & 22.023664 & 0.563209 & 0.339032 \\
                     & VAE      & 44.276279 & 0.531442 & 0.284219 \\
                     & GAN      & 74.178604 & 0.693773 & 0.080698 \\
\midrule
\multirow{3}{*}{4}   & NeuroSQL & 35.913769 & 0.538813 & 0.357857 \\
                     & VAE      & 35.430927 & 0.527985 & 0.285893 \\
                     & GAN      & 48.706165 & 0.669665 & 0.129869 \\
\midrule
\multirow{3}{*}{8}   & NeuroSQL & 37.569298 & 0.505324 & 0.353983 \\
                     & VAE      & 83.373650 & 0.516064 & 0.211640 \\
                     & GAN      & 30.621849 & 0.797056 & 0.191322 \\
\midrule
\multirow{3}{*}{16}  & NeuroSQL & 52.609642 & 0.564259 & 0.370944 \\
                     & VAE      & 88.436768 & 0.527512 & 0.161437 \\
                     & GAN      & 98.826027 & 0.620977 & 0.066925 \\
\midrule
\multirow{3}{*}{32}  & NeuroSQL & 31.695450 & 0.512806 & 0.339188 \\
                     & VAE      & 95.740288 & 0.499635 & 0.172815 \\
                     & GAN      & 46.238804 & 0.751076 & 0.072116 \\
\midrule
\multirow{3}{*}{64}  & NeuroSQL & 28.105133  & 0.532103 & 0.347139 \\
                     & VAE      & 127.357986 & 0.525766 & 0.140224 \\
                     & GAN      & 50.671597  & 0.753972 & 0.067132 \\
\midrule
\multirow{3}{*}{128} & NeuroSQL & 17.399908 & 0.591030 & 0.368674 \\
                     & VAE      & 23.283272 & 0.789889 & 0.374535 \\
                     & GAN      & 59.287437 & 0.754374 & 0.042504 \\
\bottomrule
\end{tabular}
\end{table}

\begin{table}[h]
\centering
\caption{Results on AFHQ, an animal faces dataset, using ConvNet 
— with results averaged across
latent dimensions.}
\label{tab:animals-conv256-avg}
\small
\vspace{2mm}
\begin{tabular}{@{}l r r r@{}}
\toprule
\textbf{Method} & \textbf{FID (proxy) $\downarrow$} & \textbf{LPIPS $\downarrow$} & \textbf{SSIM $\uparrow$} \\
\midrule
NeuroSQL & \textbf{32.19} & \textbf{0.544} & \textbf{0.354} \\
VAE & 71.13 & 0.560 & 0.233 \\
GAN & 58.36 & 0.720 & 0.093 \\
\bottomrule
\end{tabular}
\end{table}


\begin{table}[h]
\centering
\caption{Results on AFHQ, an animal faces dataset, using ResNet across latent dimensions. Lower is better for FID (proxy) and LPIPS; higher is better for SSIM.}
\label{tab:animals-resnet256-full}
\small
\vspace{2mm}
\begin{tabular}{@{}r l r r r@{}}
\toprule
\textbf{Latent dimension} & \textbf{Method} & \textbf{FID (proxy) $\downarrow$} & \textbf{LPIPS $\downarrow$} & \textbf{SSIM $\uparrow$} \\
\midrule
\multirow{3}{*}{2}   & NeuroSQL &  7.267559 & 0.459967 & 0.267403 \\
                     & VAE      & 15.259568 & 0.448991 & 0.235907 \\
                     & GAN      & 39.100639 & 0.550003 & 0.091928 \\
\midrule
\multirow{3}{*}{4}   & NeuroSQL & 11.560558 & 0.460007 & 0.246023 \\
                     & VAE      & 13.387914 & 0.432443 & 0.192707 \\
                     & GAN      & 37.873501 & 0.356168 & 0.120893 \\
\midrule
\multirow{3}{*}{8}   & NeuroSQL & 13.611592 & 0.471643 & 0.241053 \\
                     & VAE      & 11.736956 & 0.431859 & 0.204782 \\
                     & GAN      & 29.584837 & 0.486513 & 0.081718 \\
\midrule
\multirow{3}{*}{16}  & NeuroSQL & 17.987249 & 0.341935 & 0.208179 \\
                     & VAE      &  9.721145 & 0.347468 & 0.178343 \\
                     & GAN      & 17.546219 & 0.483133 & 0.107947 \\
\midrule
\multirow{3}{*}{32}  & NeuroSQL & 14.799847 & 0.438891 & 0.255200 \\
                     & VAE      &  8.249439 & 0.341844 & 0.188594 \\
                     & GAN      & 22.425539 & 0.363643 & 0.113281 \\
\midrule
\multirow{3}{*}{64}  & NeuroSQL & 10.561233 & 0.558646 & 0.252005 \\
                     & VAE      & 14.458963 & 0.404593 & 0.202796 \\
                     & GAN      & 40.478588 & 0.423956 & 0.107915 \\
\midrule
\multirow{3}{*}{128} & NeuroSQL & 11.019302 & 0.489376 & 0.261104 \\
                     & VAE      & 13.060718 & 0.353215 & 0.208177 \\
                     & GAN      & 18.462259 & 0.381437 & 0.104066 \\
\bottomrule
\end{tabular}
\end{table}

\begin{table}[h]
\centering
\caption{Results on AFHQ, an animal faces dataset, using ResNet 
— with results averaged across
latent dimensions.}
\label{tab:animals-resnet256-avg}
\small
\vspace{2mm}
\begin{tabular}{@{}l r r r@{}}
\toprule
\textbf{Method} & \textbf{FID (proxy) $\downarrow$} & \textbf{LPIPS $\downarrow$} & \textbf{SSIM $\uparrow$} \\
\midrule
NeuroSQL & 12.40 & 0.460 & \textbf{0.247} \\
VAE & \textbf{12.27} & \textbf{0.394} & 0.202 \\
GAN & 29.35 & 0.435 & 0.104 \\
\bottomrule
\end{tabular}
\end{table}


\begin{table}[h]
\centering
\caption{Results on AFHQ, an animal faces dataset, using U-Net across latent dimensions. Lower is better for FID (proxy) and LPIPS; higher is better for SSIM.}
\label{tab:animals-U-Net-full}
\small
\vspace{2mm}
\begin{tabular}{@{}r l r r r@{}}
\toprule
\textbf{Latent dimension} & \textbf{Method} & \textbf{FID (proxy) $\downarrow$} & \textbf{LPIPS $\downarrow$} & \textbf{SSIM $\uparrow$} \\
\midrule
\multirow{3}{*}{2}  & NeuroSQL & 17.626825 & 0.595854 & 0.272737 \\
                    & VAE      & 10.595321 & 0.620105 & 0.279837 \\
                    & GAN      & 23.709433 & 0.405830 & 0.116950 \\
\midrule
\multirow{3}{*}{3}  & NeuroSQL & 10.193194 & 0.617973 & 0.278050 \\
                    & VAE      & 29.156761 & 0.578171 & 0.173115 \\
                    & GAN      & 17.098032 & 0.607903 & 0.074634 \\
\midrule
\multirow{3}{*}{4}  & NeuroSQL & 10.773172 & 0.561935 & 0.258762 \\
                    & VAE      & 14.463408 & 0.576194 & 0.117610 \\
                    & GAN      & 13.078835 & 0.626831 & 0.057038 \\
\midrule
\multirow{3}{*}{8}  & NeuroSQL & 11.994763 & 0.638746 & 0.275091 \\
                    & VAE      & 42.187984 & 0.571152 & 0.062857 \\
                    & GAN      & 75.077110 & 0.618699 & 0.039482 \\
\midrule
\multirow{3}{*}{16} & NeuroSQL & 10.660514 & 0.659894 & 0.266318 \\
                    & VAE      & 51.288338 & 0.550339 & 0.135791 \\
                    & GAN      & 97.593185 & 0.703616 & 0.010010 \\
\midrule
\multirow{3}{*}{32} & NeuroSQL & 16.301731 & 0.595429 & 0.252468 \\
                    & VAE      & 28.342377 & 0.566271 & 0.076232 \\
                    & GAN      & 25.781809 & 0.604119 & 0.025050 \\
\midrule
\multirow{3}{*}{64} & NeuroSQL & 9.855629  & 0.660927 & 0.276445 \\
                    & VAE      & 62.289070 & 0.585000 & 0.093420 \\
                    & GAN      & 99.730621 & 0.702493 & 0.010348 \\
\bottomrule
\end{tabular}
\end{table}

\begin{table}[h]
\begin{center}
\caption{Results on AFHQ, an animal faces dataset, using U-Net 
— with results averaged across
latent dimensions $\{2,3,4,8,16,32,64\}$.}
\label{tab:animals-U-Net-method-avg}
\small
\vspace{2mm}
\begin{tabular}{@{}l r r r@{}}
\toprule
\textbf{Method} & \textbf{FID (proxy) $\downarrow$} & \textbf{LPIPS $\downarrow$} & \textbf{SSIM $\uparrow$} \\
\midrule
NeuroSQL & \textbf{12.486547} & 0.618680 & \textbf{0.268553} \\
VAE & 34.046180 & \textbf{0.578176} & 0.134123 \\
GAN & 50.295575 & 0.609927 & 0.047645 \\
\bottomrule
\end{tabular}
\end{center}
\break
\footnotesize On average, NeuroSQL achieves the best FID (proxy) and SSIM; VAE attains the lowest LPIPS.
\end{table}


\begin{table}[h]
\centering
\caption{Ablation results on MNIST, a database of handwritten digits (all runs). Lower is better for FID (proxy) and LPIPS; higher is better for SSIM.}
\label{tab:mnist-ablation-full}
\small
\vspace{2mm}
\begin{tabular}{@{}r r l r r r@{}}
\toprule
\textbf{Latent dimension} & \textbf{Seed} & \textbf{Method} &
\textbf{FID (proxy) $\downarrow$} & \textbf{LPIPS $\downarrow$} & \textbf{SSIM $\uparrow$} \\
\midrule
\multirow{3}{*}{2} & \multirow{3}{*}{11} & NeuroSQL & 0.696835 & 0.039503 & 0.564344 \\
                   &                      & VAE      & 1.070473 & 0.058065 & 0.200167 \\
                   &                      & GAN      & 2.157639 & 0.054278 & 0.245780 \\
\midrule
\multirow{3}{*}{3} & \multirow{3}{*}{11} & NeuroSQL & 0.527451 & 0.030257 & 0.668574 \\
                   &                      & VAE      & 1.443453 & 0.060009 & 0.157831 \\
                   &                      & GAN      & 1.849820 & 0.057282 & 0.220785 \\
\bottomrule
\end{tabular}
\end{table}

\begin{table}[t]
\begin{center}
\caption{Results on MNIST, a database of handwritten digits — averaged over latent dimensions $\{2,3\}$ (seed $=11$).}
\label{tab:mnist-method-avg}
\small
\vspace{2mm}
\begin{tabular}{@{}l r r r@{}}
\toprule
\textbf{Method} & \textbf{FID (proxy) $\downarrow$} & \textbf{LPIPS $\downarrow$} & \textbf{SSIM $\uparrow$} \\
\midrule
NeuroSQL & \textbf{0.612143} & \textbf{0.034880} & \textbf{0.616459} \\
VAE & 1.256963 & 0.059037 & 0.178999 \\
GAN & 2.003730 & 0.055780 & 0.233283 \\
\bottomrule
\end{tabular}
\end{center}
\break
\footnotesize NeuroSQL is best on all three metrics (FID (proxy), LPIPS, and SSIM).
\end{table}

\end{document}